\documentclass[a4paper,reqno]{amsart}

\usepackage{amsmath, amssymb, amsthm, eucal}
\usepackage{graphicx}
\usepackage[shortlabels]{enumitem}

\usepackage[below]{placeins} 
\usepackage{subcaption}
\usepackage{sidecap}
\usepackage{wrapfig}
\usepackage{float}

\usepackage[foot]{amsaddr}

\usepackage{hyperref}
\usepackage{verbatim}
\usepackage{pdfsync}

\usepackage{comment}

\usepackage{tikz}
\usepackage{scalefnt}
\usepackage[textheight=630pt,
textwidth=468pt,
centering]{geometry}

\usepackage{xcolor}

\newtheorem{theorem}{Theorem}

\newtheorem{lemma}[theorem]{Lemma}
\newtheorem{proposition}[theorem]{Proposition}

\newtheorem{remark}[theorem]{Remark}

\newtheorem{definition}[theorem]{Definition}

\numberwithin{equation}{section}
\numberwithin{theorem}{section}

\newcommand{\NN}{\ensuremath{\mathbb{N}}}

\newcommand{\RR}{\mathbb{R}}

\DeclareMathOperator{\tr}{tr}

\begin{document}

	\title[Gradient descent for deep linear networks]
	{Convergence of gradient descent for learning linear neural networks} 

\date{\today}


\author{Gabin Maxime Nguegnang$^1$}
\address{$^1$Chair for Mathematics of Information Processing, RWTH Aachen University, Pontdriesch 10, 52062 Aachen, Germany}
\email{nguegnang@mathc.rwth-aachen.de}

\author{Holger Rauhut$^1$}
\email{rauhut@mathc.rwth-aachen.de}

\author{Ulrich Terstiege$^1$}
\email{terstiege@mathc.rwth-aachen.de}

\maketitle

\begin{abstract}
	We study the convergence properties of gradient descent for training deep linear neural networks, i.e., deep matrix factorizations, by extending a previous analysis for the related gradient flow. 
	We show that under suitable conditions on the step sizes gradient descent  converges to a critical point of the loss function, i.e., the square loss in this article. Furthermore, we demonstrate that for almost all initializations gradient descent converges to a global minimum in the case of two layers. In the case of three or more layers we show that gradient descent converges to a global minimum on the manifold matrices of some fixed rank, where the rank cannot be determined a priori.
\end{abstract}



\section{Introduction}
Deep learning is arguably the most widely used and successful machine learning method, which has lead to spectacular breakthroughs in various domains such as image recognition, autonomous driving, machine translation, medical imaging and many more. Despite its widespread use, the understanding of the mathematical principles of deep learning is still in its infancy. A particular widely open question concerns the convergence properties of commonly used (stochastic) gradient descent (SGD) algorithms for learning a deep neural network from training data: Does (S)GD always converge to a critical point of the loss function? Does it converge to a global minimum? Does the network learned via (S)GD generalize well to unseen data?

In order to approach these questions we study gradient descent for learning a deep \textit{linear} network, i.e., a network with activation function being the identity, or in other words, learning a deep matrix factorization. While linear neural networks are not expressive enough for most practical applications, 
the theoretical study of gradient descent for linear neural networks is highly non-trivial and therefore expected to be very valuable. The difficulty in deriving mathematical convergence guarantees results from the functional being non-convex in terms of the individual matrices in the factorization. We are convinced that the  
case of linear networks should be well-understood before passing to the more difficult (but more practically relevant) case of nonlinear networks. We expect that some principles (though not all) will transfer to the nonlinear case and the mathematical analysis of the linear case will provide valuable insights.

This article is a continuation of the work started in \cite{bah2019learning}, where a theoretical analysis of the gradient flow related to learning a deep linear network via minimization of the square loss has been studied. Extending earlier contributions \cite{arora2018convergence,chitour2018geometric,AroraCohenHazan2018}, it was shown in \cite{bah2019learning} that gradient flow always converges to a critical point of the square loss. Moreover, for almost all initializations it converges to a global 
minimizer in the case of two layers. It is conjectured that this result also holds for more than two layers, but currently it is only shown in \cite{bah2019learning} that for more layers, gradient flow converges to the global minimum of the loss function restricted to the matrices of some fixed rank $k$ for almost all initializations, where unfortunately the result does not allow to determine $k$ a priori. 
As another interesting discovery, \cite{bah2019learning} identifies the flow of the product matrix resulting from the gradient flow for the individual matrices in the factorization as a Riemannian gradient flow with respect to a non-trivial and explicitly given Riemannian metric on the manifold of matrices of fixed rank $k$. This result requires that at initialization the tuple of individual matrices is \textit{balanced}, a term that the authors of \cite{AroraCohenHazan2018} introduced.
It is important to note that balancedness is preserved by the gradient flow, i.e., this property is related to the natural invariant set of the flow.

In this article, we extend the convergence analysis in \cite{bah2019learning} from gradient flow to gradient descent. 
Under certain conditions on the  step sizes, we show that the gradient descent iterations converge to a critical point of the square loss function. Moreover, for almost all initializations convergence is towards a global minimum in the case of two layers, while for more than two layers we obtain the analogue of the main result in \cite{bah2019learning} that for almost all initializations the product matrix converges to a global minimum of the square loss restricted to the manifold of rank $k$ matrices for some $k$.

The extension of the analysis from the gradient flow case to gradient descent turned out to be much more involved than one might initially expect. The reason is that the gradient descent iterations do no longer satisfy exactly the invariance property related to the balancedness. This property of the gradient flow, however, was heavily used in the convergence proof in \cite{bah2019learning}. In order to circumvent this problem, we develop an induction argument inspired by the article \cite{Du2018learning}, which covers the significantly simpler special case of two layers. The induction proof tracks, in particular, how much the balancedness condition is perturbed during the iterations. In fact, such perturbations stay bounded under suitable assumptions on the step sizes.

\medskip

Learning linear networks are currently studied also in the context of the so-called implicit bias of gradient descent and gradient flows. 
It was observed and studied empirically, for instance in \cite{gunasekar2017implicit, keskar2017,yun2021a,zhang2017}, that in the context of overparameterization, where more parameters than training samples are used, learned neural networks generalize surprisingly well. This in contrast to the intuition gained from classical statistics, which would predict overfitting of the learned network, i.e., that generalization properties should be poor. In this overparameterized setting there will usually exist many neural networks that exactly interpolate the training examples \cite{zhang2017}, hence, leading to the global minimum zero for the empirical loss. In particular, a global minimizer is by far not unique. This means that the used learning algorithm, mostly variants of (stochastic) gradient descent, will induce an implicit bias on the learned neural network \cite{gunasekar2017implicit,zhang2017}. As a result, one possible explanation of the phenomenon of the good generalization property of overparameterized learned neural networks is that the implicit bias of (stochastic) gradient descent is towards to solutions of low complexity in a suitable sense, resulting in good generalization. While a theoretical analysis of this phenomenon seems difficult for nonlinear networks, first works for linear networks indicate that gradient descent leads to linear networks (factorized matrices) of low rank \cite{arora2018convergence,chou2021,geyer2019implicit,gunasekar2017implicit,razin2020implicit}, although many open questions remain. Another important role seems to be played by the random initialization, see e.g.\ \cite{du2018gradient,shamir2019exponential}. 

We expect that the convergence analysis of gradient descent performed in our paper, will be a useful tool also for the detailed analysis of the implicit bias of (stochastic) gradient descent in learning deep overparameterized neural networks.

{Convergence of the stochastic subgradient method to a critical point has been established in \cite{davis2020stochastic}. This result requires the subgradient sequence to be bounded and that the cost function should be strictly decreasing along any trajectory of the differential inclusion proceeding from a noncritical point. In  addition, the authors of \cite{davis2020stochastic} commented that the  boundedness of the iterates may be enforced by assuming that the constraint set on which the set valued map is defined is bounded or by a proper choice of a regularizer. In contrast, we do not require these conditions. We rather prove the boundedness of the gradient descent sequence and demonstrate the strong descent condition of this sequence.  For the scenario of learning deep linear networks, works done in \cite{arora2018convergence,bartlett2018gradient,elkabetz2021continuous,wu2019global,yun2021a,yun2018global,Zou2020On} study the convergence of gradient descent. The authors of  \cite{elkabetz2021continuous} provided a guarantee of convergence to global minimizers for gradient descent with random balanced near-zero initialization. 
	Their proof proceeds 
	by transfering the convergence properties of gradient flow to gradient descent.
	Based on Lojasiewicz' theorem, 
	we prove that gradient descent converges to a critical point of the square loss of deep linear networks. Then we extend the result in \cite{bah2019learning} that for almost all initializations gradient descent converges to the global minimum for networks of depth $2$. For three or more layers we prove that gradient descent converges to global minimum on a manifold of a fixed rank. The convergence result in \cite{elkabetz2021continuous} is restricted to near-zero initialization with constant step size whereas our result works for almost all initialization and not necessarily constant
	step size. 
	The authors of \cite{du2019width,hu2020provable} proved that gradient descent with Gaussian resp. orthogonal random initialization  and constant step size converges to a global minimum. The result in \cite{hu2020provable} requires that the hidden layer dimension should be greater than the dimension of the input data with orthogonal initialization and the one in \cite{du2019width} assumes that the hidden layer dimension is greater than the dimension of the output data. Compared to these results, our result is more general in the sense that it does not require these conditions  
	which exclude some important 
	models such as auto-encoders 
	where the dimensions of the intermediate layers are commonly less than the input and output dimensions. Moreover, our result does not require initialization to be close enough to a global minimum (as in e.g., \cite{arora2018convergence}), and the maximum allowed step size in Theorem 2.4 does not decay exponentially with depth (Remark 2.5.(b)). In this sense, our theorem is less restrictive.}



\medskip

Our article is structured as follows. Section~\ref{sec:main-results} introduces deep linear networks and gradient descent, recalls the recent results from \cite{bah2019learning} on gradient flows, and presents our two main results on convergence to a critical point and convergence to a global minimizer for almost all initializations. Section~\ref{sec:critical-point} provides the proof of convergence to critical points (in the sense described above), while Section~\ref{sec:almost-all} is dedicated to the proof of convergence to global minimizers. Finally, Section~\ref{sec:numerical} presents numerical experiments illustrating our results.


\subsection{Notation}

The standard $\ell_p$-norm on $\RR^d$ will be denoted by $\|x\|_p = (\sum_{j=1}^d |x_j|^p)^{1/p}$ for $1 \leq p < \infty$. We write the spectral norm on $\RR^{d \times m}$ as $\|A\| = \max_{\|x\|_2=1} \|A x\|_2 = \sigma_{\max}(A)$, where $\sigma_{\max}(A)$ is the largest singular value of $A$. Moreover, we let $\sigma_{\min}(A) = \min_{\|x\|_2=1} \|Ax\|_2$ be the smallest singular value of $A$. 
The trace of a matrix $A$
is denoted as $\operatorname{tr}(A)$ and its Frobenius norm is defined as $\|A\|_F = \sqrt{\operatorname{tr}(A^T A)} = \sqrt{ \sum_{j,k} |A_{j,k}|^2}$. We will often combine matrices $W_1,\hdots,W_N$ into a tuple $\overrightarrow{W}=(W_1,\hdots,W_N)$. We define the Frobenius inner product of two such a tuples $\overrightarrow{W}$ and $\overrightarrow{V}$ as
$\langle \overrightarrow{W}, \overrightarrow{V} \rangle_F = \sum_{j=1}^N \operatorname{tr}(W_j^T V_j)$ and the corresponding Frobenius norm as
$\|\overrightarrow{W}\|_F = \sqrt{\langle \overrightarrow{W}, \overrightarrow{W}\rangle_F} = 
\left( \sum_{j=1}^N \|W_j\|_F^2\right)^{1/2}$. The operator norm of a mapping $\mathcal{A}$ acting between tuples of matrices will be denoted as 
$\| \mathcal{A}\|_{F \to F} = \max\limits_{\|\overrightarrow{W}\|_F = 1} \left\|\mathcal{A}(\overrightarrow{W})\right\|_F$.

\section{Linear Neural Networks and Gradient Descent Analysis}
\label{sec:main-results}



A neural network is a function $f: \RR^{d_x} \to \RR^{d_y}$ of the form
\[
f(x) = f_{W_1,\hdots,W_N,b_1,\hdots,b_N}(x) = g_N \circ g_{N-1}\cdots \circ g_1(x),
\]
where the so-called layers $g_j:\RR^{d_{j-1}} \to \RR^{d_j}$ are the composition of an affine function with a componentwise activation function, i.e.,
\[
g_j(z) = \sigma(W_j z + b_j), \quad \mbox{ for } W_j \in \RR^{d_j \times d_{j-1}}, b_j \in \RR^{d_j},
\]
where $\sigma : \RR \to \RR$ applied to a vector $w \in \RR^{d_j}$ acts as $(\sigma(w))_k = \sigma(w_k)$, $k \in [d_j]$. Here, $d_0 = d_x$ and $d_N = d_y$, while $d_1,\hdots,d_{N-1} \in \mathbb{N}$ are some numbers.
Prominent examples for activation functions used in deep learning include $\sigma(t) = \operatorname{ReLU}(t) = \max\{0, t\}$ and $\sigma(t) = \tanh(t)$, but we will simply choose the identity $\sigma(t) = t$ in this article.

Learning a neural network $f=f_{W_1,\hdots,W_N,b_1,\hdots,b_N}$ consists in adapting the parameters $W_j,b_j$ based on labeled training data, i.e., pairs $(x_i,y_i)$ of input data $x_1,\hdots,x_m \in \RR^{d_x}$ and output data $y_1,\hdots,y_m \in \RR^{d_y}$ in a way that
$f_{W_1,\hdots,W_N,b_1,\hdots,b_N}(x_i) \approx y_i$ for $i \in [m]$. (Ideally, the learned neural network $f$ should generalize well to unseen data, i.e., it should predict well the label $y$ corresponding to new input data $x$. However, we will not discuss this point further in this article.)

The learning process is usually performed via optimization. Given a loss function $\ell : \RR^{d_y} \times \RR^{d_y} \to \RR_+$ (usually satisfying $\ell(y,y) = 0$), one aims at minimizing the empirical risk function 
\[
\mathcal{L}(W_1,\hdots,W_N,b_1,\hdots,b_N) = \sum_{i=1}^m \ell(f_{W_1,\hdots,W_N,b_1,\hdots,b_N}(x_i),y_i)
\]
with respect to the parameters $W_1,\hdots,W_N,b_1,\hdots,b_N$. Gradient descent and stochastic gradient descent algorithms are most commonly used for this task. A convergence analysis of these algorithms is challenging in general since due to the compositional nature of neural networks, the function $\mathcal{L}$ is not convex in general.

Due to this difficulty, we reduce to the special case of \textit{linear} neural networks in this article, i.e., we assume that $\sigma(t) = t$ is the identity and that $b_j = 0$ for all $j$. Consequently, a linear neural networks takes the form
\[
f(x) = f_{W_1,\hdots,W_N}(x)= W_N \cdots W_1 x = W x, \quad \mbox{ where } W = W_N \cdot W_{N-1} \cdots W_1.
\]
While linear networks may not be expressive enough for many applications, convergence properties of gradient descent applied to learning linear neural networks are still non-trivial to understand. We will concentrate on the square-loss $\ell(z,w) = \frac{1}{2} \|z-w\|_2^2$ here, so that our learning problem consists in minimizing
\[
L^N(W_1,\dots, W_N)= \frac{1}{2} \sum_{i=1}^m \|y_i - W_N \cdots W_1 x_i\|_2^2 = \frac{1}{2}\|Y-W_N\cdots W_1 X\|^2_F = L^1(W_N \cdots W_1)
\]
where the data matrix $X \in \RR^{d_x \times m}$ contains the data points $x_i \in \RR^{d_x}$, $i=1,\hdots,m$ as columns and likewise the matrix $Y \in \RR^{d_y \times m}$ contains the label points $y_i \in \RR^{d_y}$, $i=1,\hdots,m$. The function $L^1$ is given by
\[
L^1(W) = \frac{1}{2} \|Y - W X \|_F^2.
\]
Note that the rank of the matrix $W = W_N \cdots W_1$ is at most $r:= \min_{i=0,\hdots,N} d_i$, which is strictly smaller than $\min\{d_x,d_y\}$ if one of the ``hidden'' dimensions $d_i$ is smaller than this number.
Hence, we can also view the learning problem as the one of minimizing $L^1(W)$ under the constraint $\operatorname{rank}(W) \leq r$. Instead of directly minimizing over $W$, we choose an 
\textit{overparameterized} representation as $W=W_N \cdots W_1$ and consider gradient descent with respect to each factor $W_i$. 
While overparameterization seems to be a waste of resources at first sight, it also has certain advantages as it can even accelerate convergence \cite{arora2019implicit} (at least for $\ell_p$-losses with $p > 2$) or lead to solutions with better generalization properties \cite{zhang2017}. Moreover, we expect that understanding theory for overparameterization in linear neural network will also give insights for overparameterization in nonlinear networks, which is widely used in practice. 
While speed of convergence or implicit bias are certainly of interest on their own, we will not delve into this, but rather concentrate on mere convergence here.


We consider gradient descent for the loss function $L^N$ with  step sizes $\eta_k$,  i.e.,  
\begin{equation}\label{GD}
W_j(k+1)=W_j(k)-\eta_k \nabla_{W_j}L^N(W_1(k),\dots,W_N(k)).
\end{equation}
We further define the  matrix  $W$ at each iteration $k$ by
$$
W(k)=W_N(k)\cdots W_1(k).
$$
Before discussing gradient descent itself, let us recall previous results for the related gradient flow, which will guide the intuition for the analysis in this paper.


\subsection{Gradient flow analysis}

The gradient flow $\overrightarrow{W}(t) = (W_1(t),\hdots,W_N(t))$, $t \in \RR_+$ for the function $L^N$ is defined via the differential equation
\begin{equation}\label{def:grad-flow}
\frac{d}{dt} W_j(t) = - \nabla_{W_j} L^N(W_1(t),\hdots,W_N(t)), \quad W_j(0) = W_{j,0}  \quad j = 1,\hdots,N,
\end{equation}
for some initial matrices $W_{j,0} \in \RR^{d_j \times d_{j-1}}$. This flow represents the continuous analogue of the gradient descent algorithm and has been analyzed in \cite{arora2018convergence,AroraCohenHazan2018,chitour2018geometric,bah2019learning}.

An important invariance property of the
gradient flow \eqref{def:grad-flow} consists in the fact that the differences
\begin{equation}\label{eq:invar}
W_{j+1}^T(t)W_{j+1}(t) - W_{j}(t) W_j^T(t), \quad j=1,\hdots,N 
\end{equation}
are constant in time, see \cite{arora2018convergence,AroraCohenHazan2018,chitour2018geometric,bah2019learning}. This motivates to call a tuple $\overrightarrow{W} = (W_1,\hdots,W_N)$ \textit{balanced} if 
\begin{equation}\label{eq:balanced}
W_{j+1}^T W_{j+1} = W_j W_j^T \quad \mbox{ for all } j =1,\hdots,N.
\end{equation}
If $\overrightarrow{W}(0) = (W_{1,0},\hdots,W_{N,0})$ is balanced then $\overrightarrow{W}(t)$ is balanced for all $t \in \RR_+$ as a consequence of the invariance property. 
Note that by taking the trace on both sides of \eqref{eq:balanced}, we see that balancedness implies
$\|W_j\|_F = \|W_1\|_F$ for all $j=1,\hdots,N$.

It is useful to introduce the ``end-to-end'' matrix $W(t) = W_N(t) \cdots W_1(t)$, which describes the action of the resulting network and is the object of main interest. It was shown in \cite{AroraCohenHazan2018} that if the initial tuple $\overrightarrow{W}(0)$ (and hence $\overrightarrow{W}(t)$ for any $t \geq 0$) is balanced then the dynamics of $W(t)$ can be described without making use of the individual matrices $W_j(t)$. More precisly, it satisfies the differential equation
\begin{equation}\label{Riemann-grad-flows}
\frac{d}{dt} W(t) = - \mathcal{A}_{W(t)}( \nabla L^1(W(t)),
\end{equation}
where $\mathcal{A}_{W}: \RR^{d_x \times d_y} \to \RR^{d_x \times d_y}$ is the linear map
\[
\mathcal{A}_{W}(Z) = \sum_{j=1}^N (W W^T)^{\frac{N-j}{N}} \cdot Z \cdot (W^T W)^{\frac{j-1}{N}}. 
\]
One feature of the flow in \eqref{Riemann-grad-flows}, see \cite[Theorem 4.5]{bah2019learning}, is that the rank of $W(t)$ is constant in $t$, i.e., if $W(0) = W_N(0) \cdots W_1(0)$ has rank $r$ then the $W(t)$ stays in the manifold of rank $r$ matrices for all $t \geq 0$ (but note that the rank may drop in the limit). This property may fail for non-balanced initializations \cite[Remark 4.2]{bah2019learning}. Another interesting observation (which, however, will not be important in our article) is that \eqref{Riemann-grad-flows} can be interpreted as Riemannian gradient flow with respect to an appropriately defined Riemannian metric on the manifold of rank $r$ matrices, see \cite{bah2019learning} for all the details.

The convergence properties of the gradient flow \eqref{def:grad-flow} (in both the unbalanced and balanced case) can be summarized in the following theorems. The first one from \cite[Theorem 3.2]{bah2019learning} significantly generalizes the main result of \cite{chitour2018geometric}.

\begin{theorem}\label{thm:grad-flow-convergence} Assume that $XX^T$ has full rank. Then the flow $\overrightarrow{W}(t)$ defined by \eqref{def:grad-flow} is defined and bounded for all $t \geq 0$ and converges to a critical point of $L^N$ as $t \to \infty$. 
\end{theorem}

This result is shown via Lojasiewicz' theorem \cite{absil2005convergence}, which requires in turn to show boundedness of all components $W_i(t)$ of $\overrightarrow{W}(t)$. While boundedness is straightforward to show for $W(t)$, it is a nontrivial property of the $W_i(t)$. In fact, the proof exploits the invariance of the differences in \eqref{eq:invar}.

While convergence to a critical point is nice to have, we would like to obtain more information about the type of critical point; whether it is a global or local minimum or merely a saddle point. Note that the function $L^N$ built from the square loss has the nice (but rare) property that a local minimum is automatically a global minimum \cite{kawag16,TragerKohnBruna19}. This means that we only need to single out saddle points. Also observe that we cannot expect to have convergence to a global minimizer for any initialization because the flow will not move when initilizing in any critical point, so that we cannot expect convergence to a global minimizer if that critical point is not already a global minimizer. The following result valid for almost all initializations was derived in \cite[Theorem 6.12]{bah2019learning}. In order to state it we need to introduce the matrix
\begin{equation}\label{def:Qmatrix}
Q = Y X^T (X X^T)^{-1/2},
\end{equation}
assuming that $XX^T$ has full rank. 

\begin{theorem}\label{thm:almostall-init} Assume that $XX^T$ has full rank, let $q = \operatorname{rank}(Q)$, $r = \min_{j=0,\hdots,N} d_j$ and $\bar{r} = \min\{q,r\}$.
\begin{itemize}
	\item[(a)] For almost all initializations $\overrightarrow{W}(0)$, the flow \eqref{def:grad-flow} converges to a critical point $\overrightarrow{W}^*=(W_1^*,\hdots,W_N^*)$ of $L^N$ such that $W^* := W_N^* \cdots W_1^*$ is a global minimizer 
	of $L^1$ on the manifold of matrices of fixed rank $k$ for some $0 \leq k \leq \bar{r}$.  
	\item[(b)] If $N=2$ then for almost all initial values $W_1(0),\hdots,W_N(0)$, the flow converges to a global minimizer of $L^N$ on $\RR^{d_0 \times d_1} \times \cdots \times \mathbb{R}^{d_{N-1} \times d_N}$.  
\end{itemize}
\end{theorem}
We conjecture that the statement in part (b) also holds for $N \geq 3$, or in other words, that we can always choose the maximal possible rank $k = \bar{r}$ in (a), but unfortunately, the proof method employed in \cite{bah2019learning} is not able to deliver this extension without making significant adaptations. In fact, the proof relies on an abstract result, see \cite{Lee19} and \cite[Theorem 6.3]{bah2019learning}, which states that for almost all initializations so-called strict saddle points are avoided as limits. Unfortunately, if $N \geq 3$ then minimizers of $L^1$ restricted to the manifold of matrices of rank $k < \bar{r}$ may correspond to non-strict saddle points of $L^N$, see \cite{kawag16} and \cite[Proposition 6.10]{bah2019learning}, so that the abstract result does not apply to these points.

\subsection{Gradient descent analysis}

Our main goal is to extend Theorems~\ref{thm:grad-flow-convergence} and \ref{thm:almostall-init} from gradient flow \eqref{def:grad-flow} to gradient descent \eqref{GD}. While the balancedness or more generally the invariance property, see \eqref{eq:invar}, does not appear explicitly in the statements of these theorems for gradient flow (although the invariance property is used in the proof of Theorem~\ref{thm:grad-flow-convergence}), it turns out that it does play a role in the conditions for the step sizes ensuring convergence. Unfortunately, the invariance of the differences in \eqref{eq:invar} does not carry over to the iterations of gradient descent. As a consequence, we cannot simply adapt the proof strategy in \cite{bah2019learning} for the gradient flow case. However, we will prove that under suitable conditions on the step sizes the differences in \eqref{eq:invar} will stay bounded in norm, which then allows to show boundedness of the components $W_j(k)$ of $\overrightarrow{W}(k)$ and to apply Lojasiewicz' theorem to show convergence to a critical point. 

In order to state our main results we introduce the following definition.

\begin{definition} 
We say that a tuple $\overrightarrow{W} = (W_1,\hdots,W_N)$ 
has \textbf{balancedness constant} $\delta \geq 0$ 
if
\begin{align}\label{balanced-bound1}
	\|W^T_{j+1}  W_{j+1} -W_{j}W_j^T\| &\leq \delta \qquad 
	\text{ for all } j=1,\dots, N-1.
\end{align}
\end{definition}
Obviously, \eqref{balanced-bound1} quantifies how much the tuple $\overrightarrow{W}$ deviates from being balanced, measured in the spectral norm. Note that the authors of
\cite{arora2018convergence} introduced a very similar notion and said $\overrightarrow{W} = (W_1,\hdots,W_N)$ to be $\delta$-balanced if \eqref{balanced-bound1} holds with the spectral norm replaced by the Frobenius norm.



\begin{theorem}\label{conv_critical_pt}
Let $X \in \RR^{d_x \times m}, Y \in \RR^{d_y \times m}$ be data matrices such that $X X^T$ is of full rank.
Suppose that the initialization $\overrightarrow{W}(0)$ of the gradient descent iterations \eqref{GD} has balancedness constant $\alpha \delta$ for some $\delta > 0$ and $\alpha \in [0,1)$.
Assume that the stepsizes $\eta_k > 0$ satisfy $\displaystyle\sum_{k=0}^\infty \eta_k = \infty$ and 
\begin{equation}\label{cond:stepsize:convergence:new}
	\eta_k \leq \frac{2(1-\alpha) \delta}{4 L^N(\overrightarrow{W}(0)) + (1-\alpha) \delta B_\delta} \quad
	\mbox{ for all } k \in \mathbb{N}_0,
	%
	%
	%
\end{equation}
where
\begin{align}
	B_\delta&:= 2e N K_\delta^{N-1} 
	\|X\|^2 + \sqrt{e} N K_\delta^{\frac{N}{2}-1} 
	\|XY^T\|, \label{def:B-delta}\\
	K_\delta & := M^{\frac{2}{N}} + (N+1)^2 \delta, \\
	M & := \label{M:def} \frac{\sqrt{2L^N\big(\overrightarrow{W}(0)\big)}+\|Y\|}{\sigma_{\min}{\left(X\right)}} = \frac{\sqrt{2} \|Y - W_N(0) \cdots W_1(0) X\|_F+\|Y\|}{\sigma_{\min}{\left(X\right)}}. 
\end{align}
Then  the sequence $\overrightarrow{W}(k)$ 
converges to a critical point of $L^N$.
\end{theorem}

\begin{remark}\label{rem1}
	\begin{enumerate}
		\item[(a)] If $\overrightarrow{W}(0)$ is balanced, i.e., has balancedness constant $0$, we can choose $\alpha = 0$ above. 
		Then, for any $\delta > 0$, choosing the stepsizes $\eta_k$ such that \eqref{cond:stepsize:convergence} below is satisfied
		ensures convergence to a critical point and that all the iterates $\overrightarrow{W}(k)$, $k \in \mathbb{N}$, have balancedness constant $\delta$, see Proposition~\ref{prop1}. This latter property will be a crucial ingredient for the proof of the theorem.
		\item[(b)]
		Intuitively, the step sizes $\eta_k$ should be chosen as large as possible in order to have fast convergence in practice, while it does not seem to be crucial to have the balancedness constant $\delta$ as small as possible during the iterations. This suggests to maximize the right hand side of \eqref{cond:stepsize:convergence:new} with respect to $\delta$ in order to make the condition on the stepsizes as weak as possible. While the analytical maximization seems difficult, this may be done numerically in practice. 
		A reasonably good choice for $\delta$ seems to be 
		\[
		\delta = \frac{1}{N(N+1)^2} M^{\frac{2}{N}}.
		\]
		Then $K_\delta = \left(1+ \frac{1}{N}\right) M^{\frac{2}{N}}$ so that
		$K_\delta^{N} \leq e M^2$. Since $2L^N(\overrightarrow{W}(0)) \leq \sigma^2_{\min}(X) M^2$, Condition \eqref{cond:stepsize:convergence:new} is then satisfied if
		\[
		\eta_k \leq \frac{2}{2(1-\alpha)^{-1}N(N+1)^2 M^{2-\frac{2}{N}} \sigma_{\min}^2(X) + 2e^{2-\frac{1}{N}}N M^{2-\frac{2}{N}} \|X\|^2 + e^{1-\frac{1}{N}} N  M^{1-\frac{2}{N}} \|XY^T\| }.
		\]
		In particular, the required bound does not decrease exponentially with $N$.
		\item[(c)] The stepsizes $\eta_k$ in the theorem can be chosen a priori, for instance $\eta_k=\eta$ (constant step size), or $\eta_k = c k^{-\alpha}$ for some $\alpha \in [0,1)$, or adaptively, i.e., depending on the current iterate $\overrightarrow{W}(k)$, as long as the step size condition \eqref{cond:stepsize:convergence:new} is satisfied. In practice, it seems that a large constant step size leads to best performance in terms of convergence speed.
	\end{enumerate}
\end{remark}

Of course, more information on the type of critical point to which $\overrightarrow{W}(k)$ converges is desirable. Our next theorem states the analogue of Theorem~\ref{thm:almostall-init} that essentially convergence is towards global minimizers for almost all initializations. Since Condition \eqref{cond:stepsize:convergence} on the stepsizes $\eta_k$ ensuring mere convergence to a critical point depends on the initialization $\overrightarrow{W}(0)$, we can only expect to state a result for almost all initializations for sets of tuples $\overrightarrow{W}$ of matrices
for which the balancedness constant $\delta$ and $M$ in \eqref{cond:stepsize:convergence} have a uniform upper bound. Consequently, we choose $\mathcal B\subset \RR^{d_0\times d_1}\times\cdots\times\RR^{d_{N-1}\times d_N}$ to be bounded and let
\begin{align}\label{def:deltaB}
	\delta_{\mathcal B} &= \underset{\overrightarrow{W}\in \mathcal B}{\sup}
	\max_{j=1,\hdots,N-1} \| W_{j+1}^T W_{j+1} - W_j W_j^T\|, \\
	L_{\mathcal{B}} & = \underset{\overrightarrow{W}\in \mathcal B}{\sup} L^N(\overrightarrow{W}), \quad 
	\label{def:MB}
	M_{\mathcal B} =
	\Big(\sqrt{2 L_{\mathcal{B}}} +\|Y\|\Big)\sigma^{-1}_{\min}{\left(X\right)}.
\end{align}
Note that $\delta_{\mathcal{B}}$ and $M_{\mathcal B}$ are finite  (assuming $XX^T$ has full rank) 
since $L^N$ is continuous. Let us also recall the definition of the matrix $Q = YX^T (XX^T)^{-1/2}$
in \eqref{def:Qmatrix}.

\begin{theorem}\label{thm:GD-allmost-all}
	Let $\mathcal B\subset \RR^{d_0\times d_1}\times\cdots\times\RR^{d_{N-1}\times d_N}$
	be a bounded set with constants $\delta_{\mathcal B} \leq \alpha \delta$ as in \eqref{def:deltaB} for some $\delta > 0$ and $\alpha \in [0,1)$ and $L_{\mathcal{B}}$, $M_{\mathcal B}$ defined by \eqref{def:MB}. 
	Let $q=rank(Q)$, $r=\min\{d_0,\dots,d_N\}$ and $\bar{r}=\min\{q,r\}$ and let
	$(\eta_k)_{k\in \NN_0}$ be a sequence of positive stepsizes
	such that
	\begin{equation}\label{cond:stepsize:convergence:2}
		\eta_k \leq \frac{2(1-\alpha) \delta}{4 L_{\mathcal{B}} + (1-\alpha) \delta B_\delta} \quad
		\mbox{ for all } k \in \mathbb{N}_0,
	\end{equation}
	where
	\begin{align*}
		K_\delta  := M_{\mathcal B}^{\frac{2}{N}} + (N+1)^2 \delta, \qquad
		B_\delta:= 2e N K_\delta^{N-1} 
		\|X\|^2 + \sqrt{e} N K_\delta^{\frac{N}{2}-1} 
		\|XY^T\|.
	\end{align*}
	Assume that additionally one of the following conditions is satisfied.
	\begin{itemize}
		\item[(1)] The sequence $(\eta_k)$ is constant, i.e., $\eta_k = \eta$ for some $\eta > 0$ for all $k \in \mathbb{N}$.
		\item[(2)] It holds
		$$
		\eta_k \geq C \frac{1}{k} \quad \mbox{ for some } C > 0 \quad \mbox{ and } \lim_{k \to \infty} \eta_k = 0.
		$$
	\end{itemize}
	Then the following statements hold.
	\begin{itemize}
		\item[(a)] 
		For almost all initializations $\overrightarrow{W}(0) =(W_1(0),\dots, W_N(0))\in \mathcal B$, 
		gradient descent \eqref{GD} with step sizes $\eta_k$ converges to a critical  point $\overrightarrow{W}$ of $L^N$ such that $W=W_N\cdots W_1$ is a global minimum of $L^1$ on the manifold $\mathcal{M}_k$ of matrices of rank $k=\operatorname{rank}(W)\in \{0,1,\dots,\bar{r}\}$ on $\RR^{d_N\times d_0}$.
		\item[(b)] For $N=2,$ gradient descent \eqref{GD} 
		converges to a global minimum of $L^N$ on $\RR^{d_0\times d_1}\times\RR^{d_{1}\times d_2}$ for almost all $\overrightarrow{W}(0)=(W_1(0), W_2(0))\in \mathcal B$.
	\end{itemize}	
\end{theorem}

Similar as for Theorem~\ref{thm:almostall-init}, we conjecture that part (b) extends to $N \geq 3$ or equivalently that part (a) holds with $k = \bar{r}$. As for Theorem~\ref{thm:almostall-init}, the current proof method based on a strict saddle point analysis cannot be extended to show this conjecture.

It is currently not clear, whether the theorem holds under more general assumptions on the step sizes $\eta_k$, i.e., whether it is necessary that one of the two additional conditions on $\eta_k$ holds. The current proof can only handle those two cases, for corresponding abstract results are available, see \cite{Lee19}, \cite{Panageas2019firstorder}. It seems crucial for these general results that the stepsizes are chosen a priori and independently of the choice of $\overrightarrow{W}(0)$ (or the further iterates). In particular, adaptive step size choices are not covered by our theorem.

\section{Convergence to critical points}
\label{sec:critical-point}

We will prove Theorem~\ref{conv_critical_pt} in this section. For $\overrightarrow{W} = (W_1,\hdots,W_N)$ will always denote the corresponding product matrix by
\[
W = W_N \cdots W_1,
\]
and similarly, we denote by $W(k)=W_N(k) \cdots W_1(k)$ the sequence of product matrices associated to a sequence $\overrightarrow{W}(k) = (W_1(k),\hdots,W_N(k))$, $k \in \mathbb{N}_0$.
We recall from \cite{AroraCohenHazan2018, arora2018convergence, chitour2018geometric, bah2019learning} that 
\begin{align}
	\nabla L^1\left(W\right)&=WXX^T-YX^T,  \label{Gradient1}\\
	\nabla_{W_j} L^N(W_1,\hdots, W_N) &=W^T_{j+1}\cdots W^T_{N}\nabla L^1\left(W\right) W^T_{1}\cdots W^T_{j-1}.\label{GradientN}
\end{align}




\subsection{Auxiliary bounds}

We start with a useful bound for $\|W\|$ in terms of $L^1(W)$. 

\begin{lemma}\label{lem-boundW:L1}
	Assume that $XX^T$ has full rank. Then $W \in \mathbb{R}^{d_x \times d_y}$ satisfies
	\begin{equation}
		\|W\| \leq (\|Y - W X\| + \|Y\|) \sigma_{\min}^{-1}(X) \leq \left(\sqrt{2 L^1(W)} + \|Y\|\right)\sigma_{\min}^{-1}(X).
	\end{equation}
	Consequently, if $L^N(\overrightarrow{W}(k)) \leq L^N(\overrightarrow{W}(0))$, then 
	\[
	\|W(k)\| = \|W_N(k) \cdots W_1(k)\| \leq \left(\sqrt{2 L^N(\overrightarrow{W}(0))} + \|Y\|\right)\sigma_{\min}^{-1}(X). 
	\]
	Furthermore,
	\begin{equation}\label{L1-gradient:bound}
		\|\nabla L^1(W)\| \leq \|WX - Y\|\, \|X\| \leq \sqrt{2 L^1(W)} \|X\|.
	\end{equation}
\end{lemma}
\begin{proof} Arguing similarly as in  the proof of \cite[Theorem 3.2]{bah2019learning} gives
	\begin{align*}
		\|W \| & = \|W XX^T(XX^T)^{-1}\| \leq \|WX \| \|X^T(XX^T)^{-1}\| 
		\leq (\|Y-WX\| + \|Y\|) \sigma_{\min}^{-1}(X)\\
		& \leq (\|Y-WX\|_F + \|Y\|)\sigma_{\min}^{-1}(X) = \left(\sqrt{2 L^1(W)} + \|Y\|\right) \sigma_{\min}^{-1}(X).
	\end{align*}
	The second claim follows then as an easy consequence recalling that $L^1(W(k)) = L^N(\overrightarrow{W}(k))$.
	
	For the third claim we use the explicit formula \eqref{Gradient1} for the gradient of $L^1$ to conclude that
	\[
	\|\nabla L^1(W)\| = \|WXX^T - YX^T\| \leq \|WX- Y\| \|X^T\| \leq \|WX - Y\|_F \|X\| = \sqrt{2 L^1(W)} \|X\|.
	\]
	This completes the proof.
\end{proof}

A crucial ingredient in our proof is to show boundedness of all matrices $W_j(k)$, $k \in \mathbf{N}_0$. While boundedness for the product $W(k)=W_N(k) \cdots W_1(k)$ follows easily from the previous lemma, it does not immediately imply boundedness of all the factors $W_j(k)$. For instance, multiplying one factor $W_j(k)$ by a constant $\alpha > 0$ and another factor $W_\ell(k)$ by $\alpha^{-1}$ leaves the product $W(k)$ invariant but changes the norm of $W_j(k)$ and $W_\ell(k)$. In particular, letting $\alpha \to \infty$ shows that a bound for $W(k)$ alone does not imply boundedness for $W_j(k)$, $k \in \mathbb{N}_0$. This is where the balancedness comes in. In particular, if a tuple $\overrightarrow{W}=(W_1,\hdots,W_N)$ has balancedness constant $\delta \geq 0$, then we can bound $\|W_j\|$, $j=1,\hdots,N$, by an expression (continuously) depending on $\|W\|$.  This is the essence of the next statement.

\begin{proposition}\label{Wj-bound-balancedness} Let $\overrightarrow{W}=(W_1,\hdots,W_N) \in \RR^{d_0\times d_1}\times\cdots\times\RR^{d_{N-1}\times d_N}$ with balancedness constant $\delta \geq 0$ and let $W= W_N \cdots W_1$. Then
	\[
	\|W_j\|^2 \leq \|W\|^{\frac{2}{N}} + (N+1)^2 \delta \quad \mbox{ for all } j = 1,\hdots, N. 
	\]
\end{proposition}
\begin{remark}\label{Wj-bound-balancedness-improved}
	With a significantly longer proof, one can improve this result to $$\|W_j\|^2 \leq \|W\|^{\frac{2}{N}} + N^2 \delta  \mbox{ for all } j = 1,\hdots, N.$$ However, since this does not significantly improve our results, we decided to present the slightly weaker bound in order to keep the proof short. 
\end{remark}
\begin{proof} We will first prove that
	\begin{equation}\label{ineq-W1-W}
		\|W_1\|^{2N} \leq \|W\|^2 + Q_{N,\delta}(\|W_1\|^2+\delta), 
	\end{equation}
	where $Q_{N,\delta}$ is the polynomial of degree $N-1$ defined as
	\[
	Q_{N,\delta}(x) = x(x+\delta)(x+2\delta) \cdots (x+(N-1)\delta) - x^N.
	\]
	In order to prove this claim, we let $D_j := W_{j-1} W_{j-1}^T - W_j^T W_j$ for $j=2,\hdots,N$ and note that $\|D_j\| \leq \delta$ by assumption. Moreover, 
	\begin{equation}\label{Wj-balanced}
		\|W_{j}\|^2 = \|W_j^T W_j\| = \|W_{j-1} W_{j-1}^T - D_j\| \leq \|W_{j-1}\|^2 + \delta, \quad \mbox{ for all } j=2,\hdots,N,
	\end{equation}
	and consequently
	\begin{equation}\label{Wj-bound}
		\|W_{j}\|^2 \leq \|W_1\|^2 + (j-1)\delta \quad \mbox{for } j=1,\hdots,N.
	\end{equation}
	We observe that by basic properties of the spectral norm
	\begin{align}
		\|W_1\|^{2N} & = \|(W_1^T W_1)^N \| = \|W_1^T (W_1 W_1^T)^{N-1} W_1\|
		= \|W_1^T (W_2^T W_2 + D_2)^{N-1} W_1\|\notag\\
		& \leq \|W_1^T (W_2^T W_2)^{N-1} W_1\|
		+ \sum_{k=0}^{N-2} \binom{N-1}{k} \|W_1\| \|W_2^T W_2\|^{k} \|D_2\|^{N-k-1} \|W_1\|
		\label{binomial-bound}
		\\
		& \leq \|W_1^T (W_2^T W_2)^{N-1} W_1^T\|
		+ \|W_1\|^2 \left(\sum_{k=0}^{N-1} \binom{N-1}{k} \|W_2\|^{2k} \delta^{N-k-1} -  \|W_2\|^{2(N-1)}\right)\notag\\
		& = \|W_1^T W_2^T (W_2 W_2^T)^{N-2} W_2 W_1\|
		+ \|W_1\|^2 \left( (\|W_2\|^2 + \delta)^{N-1} - \|W_2\|^{2(N-1)}\right).\notag
	\end{align}
	In the first inequality, we expanded $(W_2^T W_2 + D_2)^{N-1}$ as a (matrix) polynomial in $W_2^T W_2$ and $D_2$, observing that the highest degree term is $(W_2^T W_2)^{N-1}$. Applying the triangle inequality separates this term from the rest of the polynomial. Applying the submultiplicativity of the spectral norm to all the summands and collecting terms (which now consist of commuting scalars, i.e, the spectral norms $\|W_1\|$, $\|W_2^T W_2\|$ and $\|D_2\|$) gives the sum in \eqref{binomial-bound}, where the index $k=N-1$ is left out as it was already taken care of in the first term in \eqref{binomial-bound}. 
	
	We continue in this way, replacing $(W_2 W_2^T)^{N-2}$ by $(W_3^T W_3 + D_3)^{N-2}$ and so on. Using also \eqref{Wj-bound}, we observe that similarly as above, for $j = 2,\hdots,N-1$,
	\begin{align*}
		& \|W_1^T \cdots W_{j}^T (W_j W_j^T)^{N-j} W_j \cdots W_1\|  \\  
		& \leq  \|W_1^T \cdots W_{j+1}^T (W_{j+1} W_{j+1}^T)^{N-j-1} W_{j+1} \cdots W_1\|\\
		& + \|W_j \|^2 \cdots \|W_1\|^2 \left( (\|W_{j+1}\|^2 + \delta)^{N-j} - \|W_{j+1}\|^{2(N-j)}\right) \\
		& \leq \|W_1^T \cdots W_{j+1}^T (W_{j+1} W_{j+1}^T)^{N-j-1} W_{j+1} \cdots W_1\|\\
		&+ \|W_1\|^2\left(\|W_1\|^2+\delta\right)\cdots\left(\|W_1\|^2 + (j-1)\delta\right)\left(\left(\|W_1\|^2 + (j+1 )\delta\right)^{N-j} - \left(\|W_1\|^2 + j\delta\right)^{N-j}\right)\\
		& \leq \|W_1^T \cdots W_{j+1}^T (W_{j+1} W_{j+1}^T)^{N-j-1} W_{j+1} \cdots W_1\|\\
		&+ \left(\|W_1\|^2+\delta\right)\left(\|W_1\|^2+2\delta\right)\cdots\left(\|W_1\|^2 + j\delta\right)\left(\left(\|W_1\|^2 + (j+1 )\delta\right)^{N-j} - \left(\|W_1\|^2 + j\delta\right)^{N-j}\right).
	\end{align*}
	Hereby, we have also used that the function $x \mapsto (x+\delta)^{N-j} - x^{N-j}$ is monotonically increasing in $x \geq 0$.
	With this estimate we obtain, noting below that the sum in the second line is telescoping,
	\begin{align*}
		&\|W_1\|^{2N}  \leq \|W_1^T \cdots W_N^T W_N \cdots W_1\| \\
		& + \sum_{j=1}^{N-1} \left(\|W_1\|^2+\delta\right)\left(\|W_1\|^2+2\delta\right)\cdots\left(\|W_1\|^2 + j\delta\right)\left(\left(\|W_1\|^2 + (j+1 )\delta\right)^{N-j} - \left(\|W_1\|^2 + j\delta\right)^{N-j}\right)\\
		& = \|W_N \cdots W_1\|^2 + 
		\left(\|W_1\|^2+\delta\right)\left(\|W_1\|^2+2\delta\right)\cdots\left(\|W_1\|^2 + N\delta\right) - \left(\|W_1\|^2+\delta\right)^N\\
		& = \|W\|^2 + Q_{N,\delta}(\|W_1\|^2+\delta).
	\end{align*}
	This proves the claimed inequality \eqref{ineq-W1-W}.
	
	The fact that for all $z,\alpha\in\RR$ it holds $z(z+\alpha)\leq\left(z+\frac{\alpha}{2}\right)^2$ implies that
	\begin{align}\label{polbound}
		(x+\delta)(x+2\delta)\cdots(x+N\delta) & =\left((x+\delta)(x+N\delta)\right)\cdot \left(x+2\delta)(x+(N-1)\delta)\right)\cdots \notag \\
		&\leq \left(x + \frac{N+1}{2}\delta\right)^N. 
	\end{align}
	Setting $x = \|W_1\|^2$, $a= \|W\|^2$ and $b = \frac{N+1}{2}\delta$ and combining  inequality \eqref{ineq-W1-W} and the definition of $Q_{N,\delta}$ with \eqref{polbound}, leads to $x^N \leq a + (x + b)^{N} - (x+\delta)^N$ and hence
	\begin{equation}\label{xN-bound}
		x^N \leq a + (x + b)^{N} - x^N.
	\end{equation}
	The mean-value theorem applied to the map $x \mapsto x^N$ gives
	\[
	(x+b)^N = x^N + N \xi^{N-1} b \quad \mbox{ for some } \xi \in [x, x+b].
	\]
	Hence,
	\[
	x^N \leq a + N\xi^{N-1}b \leq a + N(x+b)^{N-1} b.
	\]
	We assume now that  $a>0$ and will comment on the case $a=0$ below. Then the previous inequality implies 
	\begin{eqnarray*}
		\frac{x^N}{a}&\leq& 1+N\frac{\left(x+b\right)^{N-1}b}{a},
	\end{eqnarray*}
	which is equivalent to
	\begin{equation}\label{zN-ineq}
		\left(\frac{x}{a^{\frac{1}{N}}}\right)^N\leq 1+N\left(\frac{x}{a^{\frac{1}{N}}}+\frac{b}{a^{\frac{1}{N}}}\right)^{N-1}\frac{b}{a^{\frac{1}{N}}}.
	\end{equation}
	Setting $z=a^{-\frac{1}{N}} x$ 
	and $c=a^{-\frac{1}{N}} b$, 
	we obtain 
	$$z^N\leq 1+Nc\left(z+c\right)^{N-1}.$$
	We claim that $z\leq 1+2Nc.$ Assume on the contrary that $z> 1+2Nc$. Then \eqref{zN-ineq} gives 
	\begin{align*}
		z&\leq \frac{1}{z^{N-1}}+Nc\left(1+\frac{c}{z}\right)^{N-1}
		<1+Nc\left(1+\frac{c}{1+2Nc}\right)^{N-1}
		\leq1+Nc\left(1+\frac{1/2}{N}\right)^{N}\\
		&\leq1+Nce^{\frac{1}{2}}.
	\end{align*}
	The last inequality implies $z\leq1+2Nc$, which is a contradiction. Thus, we showed the claim that $z\leq 1+2Nc$, that is,
	$x a^{-\frac{1}{N}} \leq 1 + 2N a^{-\frac{1}{N}} b$, which is equivalent to
	\begin{equation}\label{W1boundx}
		x \leq  a^{\frac{1}{N}}+2Nb.
	\end{equation}
	The last inequality also holds in the case $a=0$, since for $a=0$ inequality \eqref{xN-bound} remains true if we replace $a$ by any positive number $\varepsilon$ and then by our reasoning above   $ x\leq \varepsilon^{\frac{1}{N}}+2Nb$. Since this is true for any $\varepsilon>0$, it follows that for $a=0$ we have $x\leq 2Nb= a^{\frac{1}{N}}+2Nb$, thus \eqref{W1boundx} also holds for $a=0$.
	
	Using the definitions of $a,b$ and $x$, we obtain from \eqref{W1boundx} that 
	\[
	\|W_1\|^2\leq \|W\|^{\frac{2}{N}}+N(N+1)\delta.
	\]
	For any $j = 1,\hdots,N$, \eqref{Wj-bound} implies then that
	\[
	\|W_j\|^2 \leq \|W_1\|^2 + (j-1)\delta
	\leq \|W\|^{\frac{2}{N}} + N(N+1)\delta + (j-1)\delta  \leq \|W\|^{\frac{2}{N}}
	+ (N+1)^2\delta.
	\]
	This completes the proof. 
\end{proof}


\subsection{Preservation of approximate balancedness}

The key ingredient to the proof of Theorem~\ref{conv_critical_pt} is the following proposition. It is a highly non-trivial extension of \cite[Lemma 3.1]{Du2018learning} from $N=2$ layers to an arbitrary number of layers. 



\begin{proposition}\label{prop1}
	Assume that $XX^T$ has full rank and that $\overrightarrow{W}(0)=(W_1(0),\hdots,W_N(0))$ has balancedness constant $\alpha\delta$ for some $\delta > 0$ and $\alpha \in [0,1)$. Assume that the positive stepsizes $\eta_k$ satisfy \eqref{cond:stepsize:convergence}. 
	Then the gradient descent iterates $\overrightarrow{W}(k) = (W_1(k),\hdots,W_N(k))$ 
	defined by \eqref{GD} satisfy, for all $k \in \NN_0$:
	\begin{itemize}
			\item[(1)] $\overrightarrow{W}(k)$ has balancedness constant $\delta$, i.e.,
			\begin{equation}\label{balanced-k}
				\|W^T_{j+1}(k)W_{j+1}(k) -W_{j}(k)W_j^T(k)\|
				\leq \delta  \quad \mbox{ for all } j=1,\hdots,N-1;
			\end{equation}
			\item[(2)] 
			$L^N\big(\overrightarrow{W}(k)\big)\leq L^N\big(\overrightarrow{W}(0)\big)$; 
			\item[(3)] 
			$\|W_j(k)\|^2\leq K_\delta = M^{\frac{2}{N}}+(N+1)^2\delta$ for $j=1,\hdots,N$;
			\item[(4)] 
			$L^N\big(\overrightarrow{W}(k)\big)	-L^N\big(\overrightarrow{W}(k+1)\big)
			\geq \sigma \eta_k\left\|\nabla L^N\big(\overrightarrow{W}(k)\big)\right\|_F^2$.
		\end{itemize}
\end{proposition}

\begin{proof}
	We will show statements (1), (2) and (3)  by induction under the condition that 
	\begin{equation}\label{cond:stepsize:convergence}
		\eta_k \leq \min\left\{ \frac{2(1-\sigma)}{B_\delta},\frac{\sigma(1-\alpha) \delta}{2 L^N(\overrightarrow{W}(0))} \right\}
		\quad \mbox{for all } k \in \mathbb{N},
	\end{equation}
	holds for some $\sigma \in (0,1)$. The choice  
	\[
	\sigma = \frac{4 L^N(\overrightarrow{W}(0))}{4 L^N(\overrightarrow{W}(0)) +  (1-\alpha)\delta B_\delta}
	\]
	reduces \eqref{cond:stepsize:convergence}
	to \eqref{cond:stepsize:convergence:new}.
	In the induction step for (2), we will show that if (3) holds for $k$, then (4) holds for $k$ as well. Below we will always denote $W(k) = W_N(k) \cdots W_1(k)$.
	
	Since $\overrightarrow{W}(0)$ has balancedness constant $\alpha \delta<\delta$ by assumption, \eqref{balanced-k} is clearly satisfied for $k=0$. Statement (2) is trivial for $k=0$. The bound in (3) follows from a direct combination of Proposition~\ref{Wj-bound-balancedness} with Lemma~\ref{lem-boundW:L1}, i.e., for $j=1,\hdots,N$,
	\[
	\|W_j(0)\|^{2} \leq \|W(0)\|^{\frac{2}{N}} + (N+1)^2\delta \leq \left(\frac{\sqrt{2 L^N(\overrightarrow{W}(0))} + \|Y\|}{\sigma_{\min}(X)}\right)^{\frac{2}{N}} + (N+1)^2  \delta  = M^{\frac{2}{N}} + (N+1)^2 \delta,
	\]
	using also the definition of $M$ in \eqref{Wj-bound-balancedness}.
	
	
	For the induction step, we assume that $(1)$, $(2)$ and $(3)$ hold for $0,1,\dots, k$ and prove that these three properties hold for  $k+1$ as well.

	\textbf{Step 1:} We first prove statement (2) for $k+1$. To do so, we will show that if statement (3) holds for $k$ then statement (4) holds for $ k$ as well. This also proves (4) once the induction for (1), (2) and (3) is completed. 	
	
	We consider the  Taylor expansion 
	\begin{align*}
		L^N\big(\overrightarrow{W}(k+1)\big)
		= &L^N\big(\overrightarrow{W}(k)\big)+\left\langle
		\nabla_{\overrightarrow{W}}L^N\big(\overrightarrow{W}(k)\big),\overrightarrow{W}(k+1)-\overrightarrow{W}(k) \right\rangle\\
		& +\frac{1}{2}\left\langle \left(
		\overrightarrow{W}(k+1)-\overrightarrow{W}(k)\right)^T\nabla^2 L^N\left(\overrightarrow{A}_{\xi}\right), \overrightarrow{W}(k+1)-\overrightarrow{W}(k)\right\rangle,
	\end{align*}
	where $\nabla L^N\big(\overrightarrow{W}(k)\big)=\begin{pmatrix}
		\nabla_{W_1}L^N\big(\overrightarrow{W}(k)\big)\\
		\vdots\\
		\nabla_{W_N}L^N\big(\overrightarrow{W}(k)\big)
	\end{pmatrix}$ and $\overrightarrow{A}_{\xi}=\left(A^1_{\xi},\dots,A^N_{\xi}\right)$ with 
	\[
	A^i_{\xi}=W_i(k)+\xi\left(W_i(k+1)-W_i(k)\right) \quad \mbox{ for some } \xi\in [0,1],~i=1,\dots,N.
	\]
	By definition of $\overrightarrow{W}(k+1)$, this Taylor expansion  can be written as
	\begin{eqnarray*}
		L^N\left(\overrightarrow{W}(k+1)\right)&=&L^N\left(\overrightarrow{W}(k)\right)-\eta_k\left\langle\nabla L^N\left(\overrightarrow{W}(k)\right),\nabla L^N\left(\overrightarrow{W}(k)\right) \right\rangle_F\\
		&&+\frac{1}{2}\eta^2_k\left\langle\nabla L^N\left(\overrightarrow{W}(k)\right), \nabla^2 L^N\left(\overrightarrow{A}_{\xi}\right)\nabla L^N\left(\overrightarrow{W}(k)\right)\right\rangle_F.
	\end{eqnarray*}
	By the Cauchy-Schwarz inequality, we obtain  
	\begin{eqnarray}\label{strN}
		L^N\left(\overrightarrow{W}(k)\right)-L^N\left(\overrightarrow{W}(k+1)\right)	&\geq&\eta_k\left\|\nabla L^N\left(\overrightarrow{W}(k)\right)\right\|_F^2 -\frac{1}{2}\eta^2_k\left\|\nabla L^N\big(\overrightarrow{W}(k)\big)\right\|_F^2\left\|\nabla^2 L^N\left(\overrightarrow{A}_{\xi}\right)\right\|_{F\to F} \nonumber\\
		&\geq& \Bigg(1-\frac{1}{2}\eta_k\left\|\nabla^2 L^N\left(\overrightarrow{A}_{\xi}\right)\right\|_{F\to F}\Bigg)\eta_k\left\|\nabla L^N\left(\overrightarrow{W}(k)\right)\right\|_F^2.
	\end{eqnarray}	
	The crucial point now is to show that $\left\|\nabla^2 L^N\left(\overrightarrow{A}_{\xi}\right)\right\|_{F\to F}$ is bounded by the constant $B_\delta$ defined in \eqref{def:B-delta}.
	By setting $\overrightarrow{\Delta}=\left(\Delta_{1},\dots, \Delta_{N}\right)$ with $\Delta_{j}\in \RR^{d_j\times d_{j-1}}$, $j=1,\hdots,N$,
	and writing $\nabla^2 L^N\left(\overrightarrow{W}\right)\left(\overrightarrow{\Delta},\overrightarrow{\Delta}\right)$ for 
	$\left\langle \overrightarrow{\Delta}, \nabla^2 L^N\left(\overrightarrow{W}\right)\overrightarrow{\Delta}\right\rangle $, 
	the quadratic form $\nabla^2 L^N\left(\overrightarrow{W}\right)\left(\overrightarrow{\Delta},\overrightarrow{\Delta}\right)$ defined by the Hessian can be written as
	\begin{eqnarray*}
		\nabla^2 L^N\left(\overrightarrow{W}\right)\left(\overrightarrow{\Delta},\overrightarrow{\Delta}\right)&=&\sum_{j=1}^{N}\sum_{i=1}^{N}\left\langle \Delta_{j}, \frac{\partial^2 L^N(\overrightarrow{W})}{\partial W_i\partial W_j}\Delta_{i} \right\rangle\\
		&=&\sum_{i=1}^{N}\left\langle \Delta_{i}, \frac{\partial^2 L^N(\overrightarrow{W})}{\partial W_i^2}\Delta_{i} \right\rangle+\sum_{j=1}^{N}\sum_{\substack{i=1 \\ i\neq j}}^{N}\left\langle \Delta_{j}, \frac{\partial^2 L^N(\overrightarrow{W})}{\partial W_i\partial W_j}\Delta_{i} \right\rangle.
	\end{eqnarray*}
	In order to compute mixed second derivates we introduce the notation
	\begin{align*}
		Q_i(\overrightarrow{W},\Delta_i) & = W_N \cdots W_{i+1} \Delta_i W_{i-1} \cdots W_1 X \\
		P_{i,j}(\overrightarrow{W}, \Delta_i, \Delta_j) & = \begin{cases} W_N \cdots W_{j+1} \Delta_j W_{j-1} \cdots W_{i+1} \Delta_i W_{i-1} \cdots W_1 & \mbox{ if } j > i. \\
			W_N \cdots W_{i+1} \Delta_i W_{i-1} \cdots W_{j+1} \Delta_j W_{j-1} \cdots W_1 & \mbox{ if } j < i,
		\end{cases}
	\end{align*}
	with the understanding that $W_{i-1} \cdot W_1 = \operatorname{Id}$ 
	for $i = 1$ and $W_N\cdots W_{i+1} = \operatorname{Id}$ for $i=N$.
	Using the first partial derivatives of $L^N$, cf.\ \eqref{GradientN}, 
	we obtain, for $i=1,\hdots, N$,
	\begin{eqnarray*}
		\left\langle \Delta_{i}, \frac{\partial^2 L^N(\overrightarrow{W})}{\partial W_i^2}\Delta_{i} \right\rangle&=&
		\langle Q_i(\overrightarrow{W}, \Delta_i), Q_i(\overrightarrow{W}, \Delta_i)\rangle = \| Q_i(\overrightarrow{W}, \Delta_i)\|_F^2.
	\end{eqnarray*}
	The mixed second order derivatives are given, for $i \neq j$, by
	\begin{align*}
		\left\langle \Delta_{i}, \frac{\partial^2 L^N(\overrightarrow{W})}{\partial W_i\partial W_j}\Delta_{j} \right\rangle & = \langle Q_i(\overrightarrow{W},\Delta_i), Q_j(\overrightarrow{W},\Delta_j)\rangle
		+ \langle L^N(\overrightarrow{W}), P_{i,j}(\overrightarrow{W}, \Delta_i,\Delta_j) \rangle.
	\end{align*} 
	This implies that 
	\begin{align*}
		\nabla^2 L^N\left(\overrightarrow{A_{\xi}}\right)\left(\overrightarrow{\Delta},\overrightarrow{\Delta}\right) & = \sum_{i=1}^N
		\|Q_i(\overrightarrow{A}_\xi,\Delta_i)\|_F^2 + \sum_{\substack{i,j = 1\\i \neq j}}^N
		\left\langle Q_i(\overrightarrow{A}_\xi,\Delta_i), Q_j(\overrightarrow{A}_\xi,\Delta_j)\right\rangle\\
		&  +
		\sum_{\substack{i,j = 1\\i \neq j}}^N \left\langle
		A_\xi XX^T - YX^T, P_{i,j}(\overrightarrow{A}_\xi, \Delta_i,\Delta_j) \right\rangle,
	\end{align*}
	where $A_\xi = A_\xi^N \cdots A_\xi^1$ as usual.
	The Cauchy-Schwarz inequality for the trace inner product together with $\|A B \|_F \leq \|A\| \|B\|_F$ for any matrices $A,B$ of matching dimensions gives, for $ i > j$,
	\begin{align*}
		& \left|\left\langle
		A_\xi XX^T - YX^T, P_{i,j}(\overrightarrow{A}_\xi, \Delta_i,\Delta_j) \right\rangle\right| \\
		& = \left| \tr((A_\xi XX^T - YX^T)^T A_\xi^N \cdots A_\xi^{i+1} \Delta_i A_\xi^{i-1} \cdots A_\xi^{j+1} \Delta_j A_\xi^{j-1} \cdots A_\xi^1) \right| \\
		& \leq \| (A_\xi XX^T - YX^T)^T A_\xi^N \cdots A_\xi^{i+1} \Delta_i \|_F \| A_\xi^{i-1} \cdots A_\xi^{j+1} \Delta_j A_\xi^{j-1} \cdots A_\xi^1\|_F\\
		& \leq \| A_\xi XX^T - YX^T\| \, \|A_\xi^N\|  \cdots \|A_\xi^{i+1}\| \, \|\Delta_i\|_F \, \|A_\xi^{i-1}\| \cdots \|A_\xi^{j+1}\| \|\Delta_j \|_F \|A_\xi^{j-1}\| \cdots \|A_\xi^1\|,
	\end{align*}
	and similarly, for $i< j$.
	Another application of the Cauchy-Schwarz inequality gives 
	$$|\langle Q_i(\overrightarrow{A}_\xi,\Delta_i), Q_j(\overrightarrow{A}_\xi,\Delta_j)\rangle| \leq \|Q_i(\overrightarrow{A}_\xi,\Delta_i)\|_F \|Q_j(\overrightarrow{A}_\xi,\Delta_j)\|_F.$$
	Consequently 
	\begin{align}
		\left|\nabla^2 L^N\left(\overrightarrow{A_{\xi}}\right)\left(\overrightarrow{\Delta},\overrightarrow{\Delta}\right)\right|
		& \leq 
		\sum_{i,j = 1}^N \| Q_i(\overrightarrow{A}_\xi,\Delta_i)\|_F \, \| Q_j(\overrightarrow{A}_\xi,\Delta_j)\|_F \nonumber\\
		& + 
		\sum_{\substack{i,j = 1\\i \neq j}}^N \|A_\xi XX^T - Y X^T\| \|\Delta_i \|_F \|\Delta_j \|_F \prod_{\substack{k=1 \\ k \neq i,j}}^N \|A_\xi^k\| \nonumber\\
		& \leq \|X\|^2 \sum_{i,j=1}^N \|\Delta_i\|_F \|\Delta_j\|_F \left(\prod_{\underset{k\neq  i}{k=1}}^{N} \|A^k_{\xi}\|\right)
		\left(\prod_{\underset{k\neq  j}{k=1}}^{N} \|A^k_{\xi}\|\right) \nonumber\\
		& + \sum_{\substack{i,j = 1\\i \neq j}}^N \|\Delta_i \|_F \, \|\Delta_j\|_F \left( \|X\|^2 \prod_{k=1}^N  \|A_\xi^k\|   + \|Y X^T\|\right) \prod_{\substack{k=1 \\ k \neq i,j}}^N \|A_\xi^k\|. \label{Hessbound}
	\end{align}
			Using the recursive definition of $W_i(k+1)$ and that $\xi \in [0,1]$ we further obtain, for $i=1,\dots,N$, 
			\begin{align*}
				\|A^i_{\xi}\|&= \|W_i(k)+\xi\left(W_i(k+1)-W_i(k)\right)\|
				\leq \|W_i(k)\|+\|W_i(k+1)-W_i(k)\| \\ 
				& = \|W_i(k)\|+\left\|\eta_k\nabla_{W_i}L^N\left(\overrightarrow{W}(k)\right)\right\|\\
				& = \|W_i(k)\| +
				\eta_k \left\|W_{i+1}^T(k)\cdots W_N^T(k)\nabla L^1(W(k))W_1(k)^T\cdots W_{i-1}(k)\right\|.
			\end{align*}
			It follows from \eqref{L1-gradient:bound} and the induction hypothesis (2) for $k$ that
			\begin{equation}\label{ineq:nabla}
				\|\nabla L^1(W(k))\| \leq \sqrt{2 L^N(\overrightarrow{W}(k))} \| X\| \leq \sqrt{2 L^N(\overrightarrow{W}(0))} \| X\|. 
			\end{equation}
			Using the induction hypothesis (3) for $k$ this gives
			\begin{align*}
				\|A^i_{\xi}\|&\leq \|W_i(k)\|+\eta_k \sqrt{2 L^N(\overrightarrow{W}(0))} \| X\| \left(\prod_{\underset{j\neq i}{j=1}}^{N}\|W_j(k)\|\right)  \\
				& \leq K_\delta^{1/2} + \eta_k \sqrt{2 L^N(\overrightarrow{W}(0))} \| X\| K_\delta^{\frac{N-1}{2}}.
			\end{align*}
			By the assumption \eqref{cond:stepsize:convergence} on the stepsize $\eta_k$ and the definitions of $K_\delta$ and $B_\delta$ 
			we have
			\begin{align*}
				\eta_k \sqrt{2L^N(\overrightarrow{W}(0))} \| X\| K_\delta^{\frac{N-1}{2}} 
				& \leq \frac{2(1-\sigma)}{B_\delta} \sqrt{2 L^N(\overrightarrow{W}(0))} \|X\| K_\delta^{\frac{N-1}{2}}
				\leq \frac{2 \sqrt{2 L^N(\overrightarrow{W}(0))} \|X\| K_\delta^{\frac{N-1}{2}}}{2eN K_\delta^{N-1} \|X\|^2} \\
				& \leq \frac{M \sigma_{\min}(X) \|X\| K_\delta^{\frac{N-1}{2}}}{eN K_\delta^{N-1} \|X\|^2}
				\leq \frac{(M^{\frac{2}{N}} + N^2 \delta)^{\frac{N}{2}} K_{\delta}^{\frac{N-1}{2}}}{eN K_{\delta}^{N-1}} \leq \frac{1}{2N} K_\delta^{\frac{1}{2}}.
			\end{align*}
			{In the first inequality of the last line we used that by definition of $M$ we have $$\sqrt{2 L^N(\overrightarrow{W}(0))}=M \sigma_{\min}(X)-\|Y\|\leq M \sigma_{\min}(X).$$}
			It follows  that
			$$
			\|A^i_{\xi}\|\leq \left(1 + \frac{1}{2N}\right) K_\delta^{1/2}. 
			$$  
			Substituting this bound into 
			\eqref{Hessbound}, we  obtain
			\begin{align*}
				& \left|\nabla^2 L^N\left(\overrightarrow{A_{\xi}}\right)\left(\overrightarrow{\Delta},\overrightarrow{\Delta}\right)\right| \\
				&\leq 
				\left(1+\frac{1}{2N}\right)^{2N-2}
				K_\delta^{N-1}
				\|X\|^2\sum_{i,j=1}^{N}\|\Delta_j\|_F\|\Delta_i\|_F  \\
				& +\left(\left(1+\frac{1}{2N}\right)^{2N-2} K_\delta^{N-1} 
				\|X\|^2+\left(1+\frac{1}{2N}\right)^{N-2} K_\delta^{N/2-1}
				\|XY^T\|\right) \sum_{j=1}^{N}\sum_{\underset{i\neq j}{i=1}}^{N}\|\Delta_j\|_F\|\Delta_i\|_F\displaybreak[2]\\
				&\leq e K_\delta^{N-1} \|X\|^2
				\left( \sum_{j=1}^N\| \Delta_j\|\right)^2+ 
				\left(e K_\delta^{N-1}
				\|X\|^2+
				e^{1/2} K_\delta^{N/2-1}\|XY^T\|\right) \left( \sum_{j=1}^N \|\Delta_j\|\right)^2\\
				& \leq \left[2e N K_\delta^{N-1} 
				\|X\|^2 + \sqrt{e} N K_\delta^{N/2-1} 
				\|XY^T\| \right] \|\overrightarrow{\Delta}\|_F^2,
			\end{align*}
			where we have used that $(1+1/(2N))^{2N} \leq e$ and that  $\sum_{j=1}^{N}\|\Delta_{W_j}\|_F \leq \sqrt{N} \|\overrightarrow{\Delta}\|_F$.
			Hence, we derived that
			\[
			\left\|\nabla^2L^N\left(\overrightarrow{A_{\xi}}\right)\right\|_{F\to F} \leq
			2e N K_\delta^{N-1} 
			\|X\|^2 + \sqrt{e} N K_\delta^{\frac{N}{2}-1} 
			\|XY^T\|
			= B_\delta.
			\]
			Substituting this estimate into \eqref{strN} and using that the step sizes satisfy \eqref{cond:stepsize:convergence} gives
			\begin{align}
				& L^N\left(\overrightarrow{W}(k)\right)	-L^N\left(\overrightarrow{W}(k+1)\right)\geq \Bigg(1-\frac{1}{2}\eta_k B_\delta\Bigg)\eta_k\left\|\nabla L^N\left(\overrightarrow{W}(k)\right)\right\|_F^2 \notag\\
				& \geq \sigma \eta_k \left\|\nabla L^N\left(\overrightarrow{W}(k)\right)\right\|_F^2 \geq 0.\label{strN2}
			\end{align} 
			This shows the statement (4) for $k$.
			It follows by the induction hypothesis (2) for $k$ that
			$$L^N \left(\overrightarrow{W}(0)\right)  \geq L^N\left(\overrightarrow{W}(k)\right) \geq L^N\left(\overrightarrow{W}(k+1)\right).$$
			This shows statement (2) for $k+1$.
			
			\textbf{Step 2:} Let us now show that statement (1) holds at iteration $k+1$. 
			For $j=1,\hdots,N-1$ we obtain
			\begin{align*}			&\|W^T_{j+1}(k+1)W_{j+1}(k+1) -W_j(k+1)W_j^T(k+1)\|\\
				&=\left\|\Big(W_{j+1}^T(k)-\eta_k\nabla^T_{W_{j+1}}L^N\big(\overrightarrow{W}(k)\big)\Big)\Big(W_{j+1}(k)-\eta_k\nabla_{W_{j+1}}L^N\big(\overrightarrow{W}(k)\big)\Big) \right.\\
				& \quad \left.-\Big(W_j(k)-\eta_k \nabla_{W_j}L^N\big(\overrightarrow{W}(k)\big)\Big)\Big(W^T_j(k)-\eta_k\nabla^T_{W_j}L^N\big(\overrightarrow{W}(k)\big)\Big)\right\|\displaybreak[2]\\
				&=\left\|W^T_{j+1}(k)W_{j+1}(k)-W_j(k)W_j^T(k)\right.\\
				& \quad +\eta_k\Big(- W_{j+1}^T(k)W_{j+2}^T(k)\dots W_N^T(k)\nabla L^1(W(k))W_1^T(k)\cdots W_j^T(k)\\
				& \qquad \qquad - W_j(k)\cdots W_1(k)\nabla^T L^1(W(k))W_N(k)\cdots W_{j+2}(k)W_{j+1}(k) \Big)\\
				&\qquad \qquad  + W_j(k)W_{j-1}(k)\cdots W_1(k)\nabla^T L^1(W(k))W_N(k)\cdots W_{j+2}(k)W_{j+1}(k)\\
				& \qquad \qquad + W_{j+1}^T(k)W_{j+2}^T(k)\cdots W_N^T(k)\nabla L^1(W(k))W_1^T(k)\cdots W_{j-1}^T(k)W_j^T(k)\Big)\\
				& \left. \quad +  \eta_k^2 \left(\nabla^T_{W_{j+1}}L^N\big(\overrightarrow{W}(k)\big)\nabla_{W_{j+1}}L^N\big(\overrightarrow{W}(k)\big)- \nabla_{W_{j}}L^N\big(\overrightarrow{W}(k)\big)\nabla^T_{W_{j}}L^N\big(\overrightarrow{W}(k)\big)\right)\right\|\displaybreak[2]\\
				%
					%
					&\leq \left\|W^T_{j+1}(k)W_{j+1}(k)-W_j(k)W_j^T(k)\right\|
					+ \eta_k^2 \left(\|\nabla_{W_{j+1}} L^N(\overrightarrow{W}(k))\|^2 +\|\nabla_{W_{j}} L^N(\overrightarrow{W}(k))\|^2 \right).
				\end{align*}
				Applying this inequality repeatedly, we obtain
				\begin{align}
					&\|W^T_{j+1}(k+1)W_{j+1}(k+1) -W_j(k+1)W_j^T(k+1)\| \notag\\
					& \leq \|W^T_{j+1}(0)W_{j+1}(0) -W_j(0)W_j^T(0)\| 
					+ \sum_{\ell=0}^k \eta_\ell^2 \left(\|\nabla_{W_{j+1}} L^N(\overrightarrow{W}(\ell))\|^2 +\|\nabla_{W_{j}} L^N(\overrightarrow{W}(\ell))\|^2\right) \notag\\
					& \leq \alpha \delta + 2 \left(\max_{\ell=0,\hdots,k} \eta_\ell\right) \sum_{\ell=0}^k \eta_\ell \|\nabla L^N(\overrightarrow{W}(\ell))\|_{F}^2, \label{aux-upper}
				\end{align}
				where we have used that $\overrightarrow{W}(0)$ has balancedness constant $\alpha \delta$ by assumption and that 
				\begin{align*}
					\|\nabla L^N(\overrightarrow{W}(k))\|_{F \to F}^2 & \geq \max_{\ell = 1,\hdots,N} \|\nabla_{W_\ell} L^N(\overrightarrow{W}(k))\|^2 \\
					& \geq \frac{1}{2} \left( \|\nabla_{W_j} L^N(\overrightarrow{W}(k))\|^2 +\|\nabla_{W_j+1} L^N(\overrightarrow{W}(k))\|^2 \right).
				\end{align*}
				The inequality \eqref{strN2} from the previous step gives
				\begin{equation}\label{aux-lower}
					L^N(\overrightarrow{W}(0)) - L^N(\overrightarrow{W}(k+1)) = \sum_{j=0}^k \left(L^N(\overrightarrow{W}(j)) - L^N(\overrightarrow{W}(j+1))\right)
					\geq \sigma \sum_{j=0}^k \eta_k \|\nabla L^N(\overrightarrow{W}(k))\|_F^2.
				\end{equation}
				Combining inequalities \eqref{aux-upper} and \eqref{aux-lower} yields
				\begin{align*}
					& \|W^T_{j+1}(k+1)W_{j+1}(k+1) -W_j(k+1)W_j^T(k+1)\| \\
					& \leq \alpha \delta + \frac{2}{\sigma} \left(\max_{\ell=0,\hdots,k} \eta_\ell\right) \left(L^N(\overrightarrow{W}(0)) - L^N(\overrightarrow{W}(k+1))\right)\\ & 
					\leq  \alpha \delta + \frac{2}{\sigma} \left(\max_{\ell=0,\hdots,k} \eta_\ell\right) L^N(\overrightarrow{W}(0)) 
					\leq \alpha \delta + (1-\alpha) \delta = \delta,
				\end{align*}
				where we have used Condition \eqref{cond:stepsize:convergence} on the stepsizes.
						This proves statement (1) for $k+1$.

						\textbf{Step 3:} 
						For the proof of statement (3) for $k+1$ we use that we have already shown that (1) and (2) hold for $k+1$.
						It follows from Proposition~\ref{Wj-bound-balancedness} and Lemma~\ref{lem-boundW:L1} that
						\[
						\|W_j(k+1)\|^2 \leq \|W(k+1)\|^{\frac{2}{N}} + (N+1)^2 \delta
						\leq \left(\frac{\sqrt{2L^N(\overrightarrow{W}(0))} + \|Y\|}{\sigma_{\min}(X)}\right)^{\frac{2}{N}} + (N+1)^2 \delta = K_\delta.
						\]
						This shows (3) for $k+1$ and completes the proof of the proposition.
					\end{proof}
					


				\subsection{Convergence of gradient descent to a critical point}  We will use a result from \cite{absil2005convergence} to prove Theorem \ref{conv_critical_pt}, which is based on the following definition.
				
				\begin{definition}[Strong descent conditions \cite{absil2005convergence}]\label{strong_descent_cond}	
					We say that a sequence $x_k\in\RR^n$ satisfies the strong descent conditions (for a differentiable function $f:\RR^n\to \RR$) if 
					\begin{align}\label{descond1}
						f(x_k)-f(x_{k+1}) & \geq\sigma\|\nabla f(x_k)\|\|x_{k+1}-x_k\|\\
						\mbox{ and } \quad \label{descond2}
						f(x_{k+1})=f(x_k) & \implies x_{k+1}=x_k
					\end{align}
					hold for some $\sigma>0$ and for all $k$ larger than some $K.$
				\end{definition}
				
				The next theorem is essentially an extension of Lojasiewicz' theorem to discrete variants of gradient flows.
				
				\begin{theorem}\cite[Theorem 3.2]{absil2005convergence}\label{absil2005}
					Let $f:\RR^n\to\RR$ be an analytic cost function. Let the sequence $\{x_k\}_{k=1,2,\dots}$ satisfy the strong descent conditions (Definition \ref{strong_descent_cond}). Then, either $\lim\limits_{k\to\infty}\|x_k\|=+\infty,$ or there exists a single point $x^*\in \RR$ such that $$\lim\limits_{k\to\infty}x_k=x^*$$
				\end{theorem}

				Now we are ready to prove Theorem \ref{conv_critical_pt}.
				
				\begin{proof}
					%
					By point (4) of Proposition~\ref{prop1} and since  $\overrightarrow{W}(k+1) - \overrightarrow{W}(k)=\eta_k\nabla L^N\left(\overrightarrow{W}(k)\right)$ for all $k\in \NN_0$, we have 
					\begin{equation}
						L^N\big(\overrightarrow{W}(k)\big)	-L^N\big(\overrightarrow{W}(k+1)\big)\geq \sigma \left\|\nabla L^N\big(\overrightarrow{W}(k)\big)\right\|_F\left\|\overrightarrow{W}(k+1) - \overrightarrow{W}(k)\right\|_F,\label{strN3}
					\end{equation}
					which means that the first part \eqref{descond1} of the strong descent condition  holds. This implies then that also the second part \eqref{descond2} of the strong descent condition holds, since 
					if $L^N\big(\overrightarrow{W}(k+1)\big)	= L^N\big(\overrightarrow{W}(k))$, it follows that $$\left\|\nabla L^N\big(\overrightarrow{W}(k)\big)\right\|_F\left\|\overrightarrow{W}(k+1) - \overrightarrow{W}(k)\right\|_F=0,$$ hence $\overrightarrow{W}(k+1) = \overrightarrow{W}(k)$ or  $\nabla L^N\big(\overrightarrow{W}(k)\big)=0$, but the latter again implies  $\overrightarrow{W}(k+1) = \overrightarrow{W}(k)$. Thus indeed $\overrightarrow{W}(k+1) = \overrightarrow{W}(k)$ if  $L^N\big(\overrightarrow{W}(k+1)\big)	= L^N\big(\overrightarrow{W}(k)$.
					
					Since by Proposition~\ref{prop1}, the sequence $(\overrightarrow{W}(k))_{k\in\NN_0}$ is bounded and $L^N$ is analytic, it follows from 
					Theorem~\ref{absil2005} that there exists $\overrightarrow{W}^*$ such that 
					\[
					\lim\limits_{k\to\infty}\overrightarrow{W}(k)=\overrightarrow{W}^*.
					\]
					
					It  remains to show that $\overrightarrow{W}^*$ is a critical point of $L^N$. 
					Since $\nabla L^N\big(\overrightarrow{W}\big)$ is continuous in $\overrightarrow{W}$, it follows that 
					$\nabla L^N\big(\overrightarrow{W}^*\big)=\lim\limits_{k\to\infty} \nabla L^N\big(\overrightarrow{W}(k)\big)$  and that  
					\[
					\|\nabla L^N\big(\overrightarrow{W}^*\big)\|_F=\lim\limits_{k\to\infty} \|\nabla L^N\big(\overrightarrow{W}(k)\big)\|_F=:c.
					\]
					
					In order to show that   $\overrightarrow{W}^*$ is a critical point, it suffices to show that  $c=0$.
					A repeated application of point (4) of Proposition~\ref{prop1} gives 
					\[
					L^N\big(\overrightarrow{W}(0)\big)	-L^N\big(\overrightarrow{W}(k+1)\big)\geq\sigma\sum_{j=0}^{k}\eta_j\left\|\nabla L^N\big(\overrightarrow{W}(k)\big)\right\|_F^2 \quad \mbox{ for any } k \in \mathbb{N},
					\]
					hence, taking the limit,
					\[
					L^N\big(\overrightarrow{W}(0)\big)\geq \sigma\sum_{k=0}^{\infty}\eta_k\left\|\nabla L^N\big(\overrightarrow{W}(k)\big)\right\|_F^2. 
					\]
					Assume now that $c\neq 0$. Then $c>0$ and  there exists $k_0\in \NN$ such that 
					\[
					\left\|\nabla L^N\big(\overrightarrow{W}(k)\big)\right\|_F\geq \frac{c}{2}~~~\forall k\geq k_0.
					\]
					But then 
					\[
					L^N\big(\overrightarrow{W}(0)\big)\geq\sigma\sum_{k=k_0}^{\infty}\eta_k\|\nabla L^N\big(\overrightarrow{W}(k)\big)\|_F^2\geq \frac{c^2}{4}\sigma\sum_{k=k_0}^{\infty}\eta_k,
					\]
					which by $\sigma > 0$ contradicts our assumption that $\displaystyle\sum_{k=0}^{\infty}\eta_k=\infty$. 
					Thus indeed $c=0$ and $\overrightarrow{W}^*$ is a critical point of $L^N$.
				\end{proof}

				\section{Convergence to a global minimum for almost all initializations}
				\label{sec:almost-all}
				


				Let us now  transfer \cite[Theorem 6.12]{bah2019learning} to our situation of the gradient descent method by showing Theorem~\ref{thm:GD-allmost-all}.
				%
				%
				Our proof is based on the following abstract theorem, 
				which basically states that gradient descent schemes
				avoid strict saddle points for almost all initializations. The case of constant step sizes (condition (1)) was shown in \cite[Proposition 1]{Lee19}, while 
				the one for step sizes converging to zero was proven in \cite[Theorem 5.1]{Panageas2019firstorder}.
				We call
				a critical point $z^*$ of a twice continously differentiable function $f$ a strict saddle point if the Hessian $\nabla^2 f (z^*)$ has at least one negative eigenvalue. Note that an analysis of the strict saddle points of $L^N$ has been performed in \cite{bah2019learning}, extending \cite{kawag16,TragerKohnBruna19}.
				
				\begin{theorem}\label{abstract-almost-sure-convergence}
					Let $f : \mathbb{R}^p \to \mathbb{R}$ be a twice continuously differentiable function and consider the gradient descent scheme 
					\[
					z(k+1) = z(k) - \eta_k \nabla f(z(k)),
					\]
					where $(\eta_k)$ satisfies one of the following conditions.
					\begin{itemize}
						\item[(1)] The sequence $(\eta_k)$ is constant, i.e., $\eta_k = \eta$ for some $\eta > 0$ for all $k \in \mathbb{N}$.
						\item[(2)] It holds
						$$
						\eta_k \geq C \frac{1}{k} \quad \mbox{ for some } C > 0 \quad \mbox{ and } \lim_{k \to \infty} \eta_k = 0.
						$$
					\end{itemize}
					Then the set of initializations $z(0) \in \mathbb{R}^p$ such that $(z(k))_k$ converges to a strict saddle point of $f$ has measure zero. 
				\end{theorem}
				Now we are ready to prove Theorem \ref{thm:GD-allmost-all}.
				
				\begin{proof}
					Due to definitions \eqref{def:deltaB}, \eqref{def:MB} of the constants $\delta_\mathcal{B}$, $L_{\mathcal{B}}$ and $M_{\mathcal{B}}$ together with condition \eqref{cond:stepsize:convergence:2} on the step sizes $\eta_k$, the conditions of Theorem~\ref{conv_critical_pt} are satisfied for each initialization $\overrightarrow{W}(0) \in \mathcal{B}$. Hence, $\overrightarrow{W}(k)$ converges to a critical point of $L^N$ for all $\overrightarrow{W}(0) \in \mathcal{B}$.
					%
					By Theorem~\ref{abstract-almost-sure-convergence}
					convergence of gradient descent with initial values in $\mathcal B$ and with step sizes $\eta_k$ to a strict saddle point occurs only for subset of $\mathcal{B}$ that has measure zero. 
					
					The rest of the proof is the same as the corresponding reasoning in  the proof of \cite[Theorem 6.12]{bah2019learning}. Let us repeat only the main aspects from \cite{bah2019learning}.
					We denote by $\overrightarrow{W}=(W_1,\hdots,W_N)$ the limit of $\overrightarrow{W}(k)$, $W= W_N\cdots W_1$ and $k = \operatorname{rank}(W)$. Then $k \leq r$ and $W$ is a critical point of $L^1$ restricted to manifold $\mathcal{M}_k$ of rank $k$ matrices \cite[Proposition 6.8(a)]{bah2019learning}. Then \cite[Proposition 6.6(1)]{bah2019learning} implies that $k \leq q$. If $W$ is not a global minimizer of $L^1$ restricted to $\mathcal{M}_k$ then $W$ is a strict saddle point of $L^N$ by \cite[Proposition 6.9]{bah2019learning}. As argued above, the set of initializations converging to such a point has measure zero, showing part (a). (Note that for $N \geq 3$ and $k < \min\{r,q\}$ a global minimizer of $L^1$ restricted to $\mathcal{M}_k$ may correspond to a non-strict saddle point $\overrightarrow{W}$ of $L^N$, see \cite[Proposition 6.10]{bah2019learning}.)
					If $N = 2$, then by \cite[Proposition 6.11]{bah2019learning} any critical point $\overrightarrow{W}=(W_1,W_2)$ such that $W=W_2 W_1$ is a global minimum of $L^1$ restricted to $\mathcal{M}_k$ for some $k<\bar r$ is a strict saddle point of $L^2$, which shows part (b) of the theorem.
				\end{proof}

				\section{Numerical experiments}
				\label{sec:numerical}
				
			
			In this section we illustrate our theoretical results 
			with numerical experiments.
			In particular, we test convergence of gradient descent for various choices of constant and decreasing step sizes and with
			$N=2$, $N=3$ and $N=5$ layers.
			
			The sample size is chosen as $m=3\cdot d$ with $d=70$. For our experiments we generate our dataset $X\in \RR^{d_x\times m}$ randomly with entries drawn from a mean zero Gaussian distribution 
			with variance $\sigma^2=1/d$,
			where $d_x=d$. The data matrix $Y \in \RR^{d \times m}$ is a random matrix of rank $r=2$, which is generated as described below. We initialize the weight matrices $W_j \in \RR^{d_j \times d_{j-1}}$
			such that $\overrightarrow{W}(0)=(W_1,\hdots,W_N)$ is balanced, i.e.,
			has balancedness constant $0$ so that $\alpha = 0$
			in Theorems~\ref{conv_critical_pt} and \ref{thm:GD-allmost-all}, in the following way. 
			The rank parameter is chosen as $r=2$ 
			and the dimensions $d_j$ 
			as 
			\[
			d_0 = d, \qquad d_1 = r , \qquad d_j=\operatorname{round}\left(r+(j-1)\frac{d-r}{N-1}\right), \quad j=2,\dots,N, 
			\]
			where $\operatorname{round}(z)$ rounds a real number $z$ to the nearest integer.
			We randomly generate
			orthogonal matrices $U_1 \in \RR^{d \times d}$, $V_j \in \RR^{d_{j} \times d_{j}}$, $j=1,\hdots,N$, according to the uniform distribution on the corresponding unitary groups and let $U_j \in \RR^{d_{j} \times d_{1}}$, $j=2,\hdots,N$ be the matrix composed of the first $d_1$ columns of $V_{j-1}$. 
			We then set
			\[
			W_j = V_j I_{d_j,d_1} U_j^T,
			\]
			where for any $n_1,n_2\in \NN$ the matrix $I_{n_1,n_2} \in \RR^{n_1\times n_2}$ is a rectangular diagonal matrix with ones on the diagonal.
			By orthogonality and construction of $U_{j+1}$, it follows that for all $j=1,\hdots,N-1$, we have
			\begin{align*}
				W_{j+1}^T W_{j+1} &= U_{j+1} I_{d_1,d_{j+1}} V_{j+1}^T V_{j+1} I_{d_{j+1},d_{1}} U_{j+1}^T = U_{j+1} U_{j+1}^T = V_j I_{d_{j},d_{1}} U_j^T U_j I_{d_{1},d_j} V_j^T \\
				&= W_j W_j^{T} 
			\end{align*}
			so that the tuple $(W_1,\hdots,W_N)$ is balanced.
			
			The random matrix $Y \in \RR^{d \times m}$ of rank $2$ is generated as $Y = \widetilde{W_N} \cdots \widetilde{W_1}X$ with matrices $\widetilde{W_j}$ generated in the same way as the matrices $W_j$. (We decided to choose a matrix $Y$ of rank $2$ so that the global minimizer of $L^1$ is also of rank $2$ and convergence to it means that $L^N$ converges to zero, which is simple to check.)
			
			\begin{figure}[t]
				\begin{subfigure}{0.45\textwidth}\centering\includegraphics[width=0.99\textwidth]{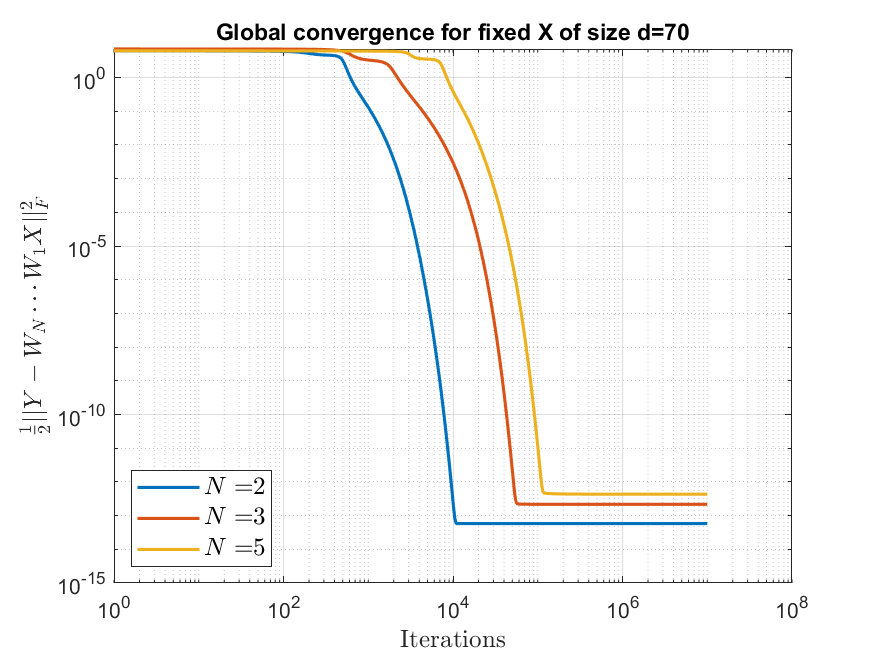}
					\caption{
						Constant step size $\eta$ meeting upper bound \eqref{bound:eta:numerical}}
					\label{fig:conv_suggested_delta_etak_cons_0,000773_0,000129_0,0000395}
				\end{subfigure}\hfill\begin{subfigure}{0.45\textwidth}
					\centering\includegraphics[width=0.99\textwidth]{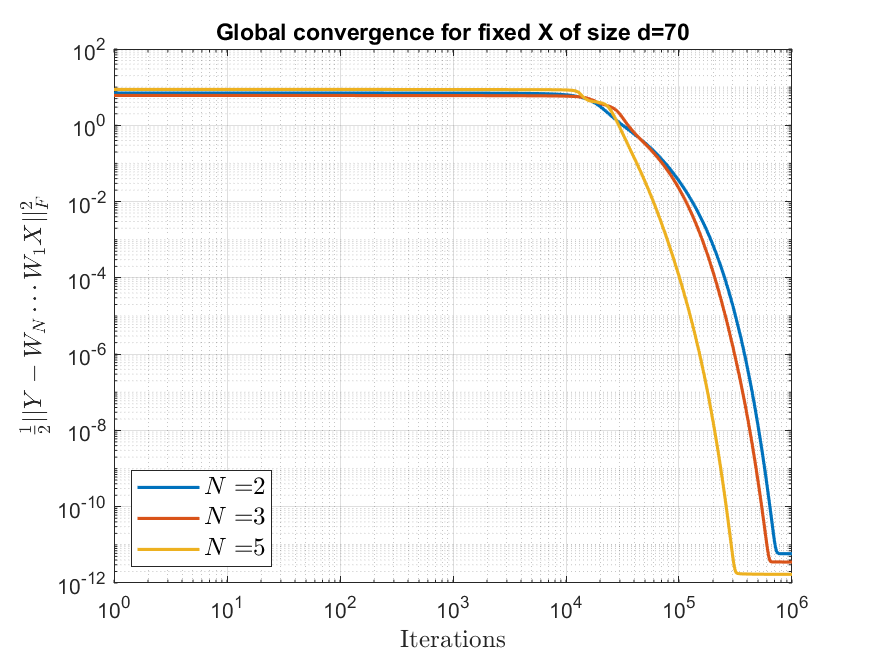}
					\caption{
						Constant step size $\eta=10^{-5}$}
					\label{fig:conv_const_suggested_delta_eta=10^(-5)}
				\end{subfigure}
				
				\begin{subfigure}{0.45\textwidth}
					\centering\includegraphics[width=0.99\textwidth]{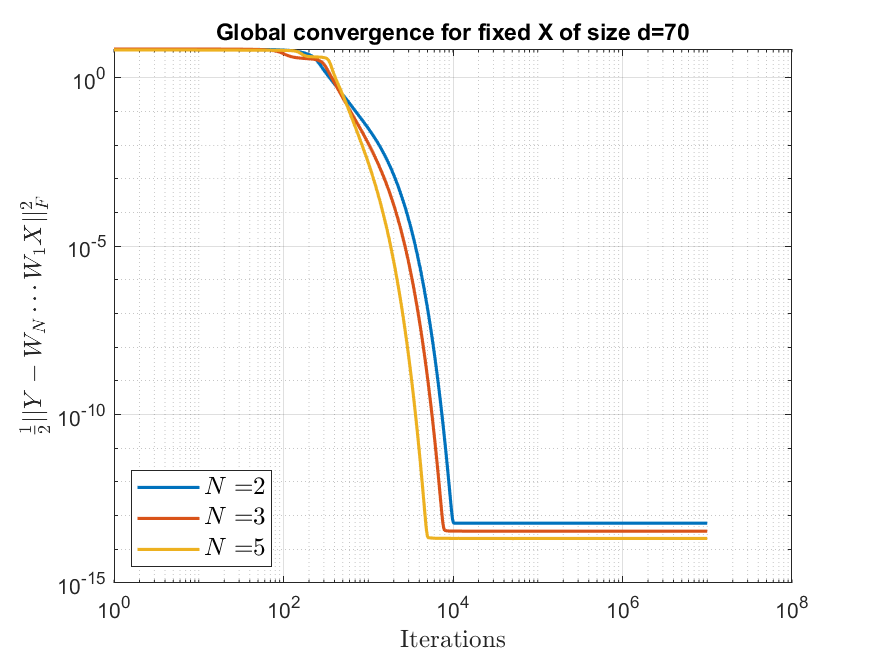}
					\caption{
						Constant step size $\eta=10^{-3}$}
					\label{fig:conv_const_suggested_delta_eta=10^(-3)}
				\end{subfigure}\hfill\begin{subfigure}{0.45\textwidth}
					\centering\includegraphics[width=0.99\textwidth]{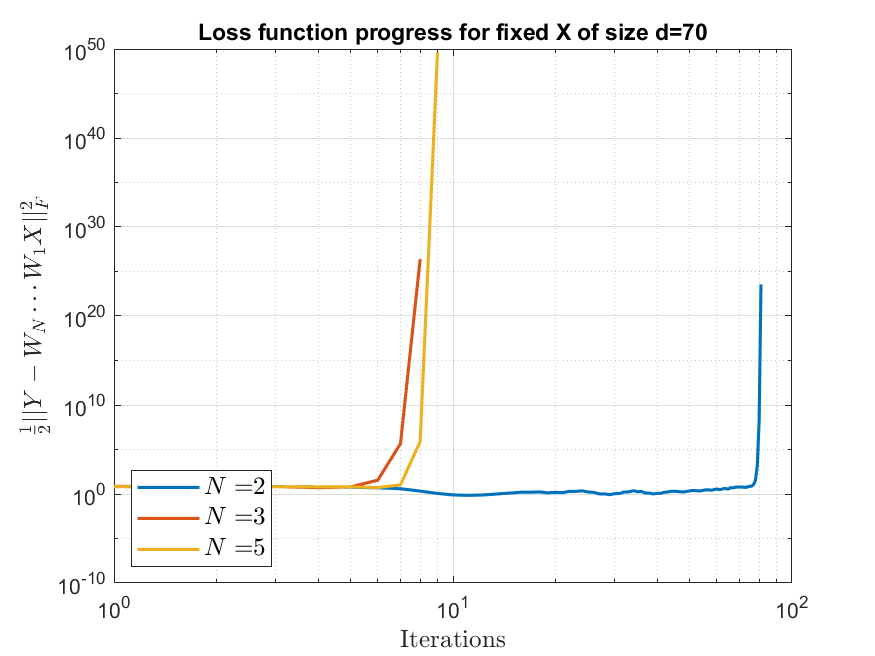}
					\caption{
						Constant step size $\eta=0.05$}
					\label{fig:conv_const_suggested_delta_eta=0.05}
				\end{subfigure}
				\caption{Progress of loss function $L^N$ for training linear networks via gradient descent for various values of the constant step size}\label{fig:const-stepsize}
			\end{figure}

			In our first set of experiments, we use a constant step size, i.e., $\eta_k = \eta$. Using $\alpha = 0$, the sufficient condition in Theorem~\ref{conv_critical_pt} reads
			\begin{equation}\label{bound:eta:numerical}
				\eta \leq \frac{2 \delta}{4 L^N(\overrightarrow{W}(0)) + \delta B_\delta},
			\end{equation}
			with $B_\delta$ in \eqref{def:B-delta}. We choose \[
			\delta = \frac{M^{\frac{2}{N}}}{N^3}.
			\]
				This slightly differs from the choice of $\delta$ suggested by Remark~\ref{rem1}(b), but  corresponds to the choice of $\delta$ that we would obtain at this point using the bound given by Remark \ref{Wj-bound-balancedness-improved} (instead of Proposition \ref{Wj-bound-balancedness}) allowing us to set $K_{\delta}= M^{\frac{2}{N}}+N^2$ (instead of $K_{\delta}= M^{\frac{2}{N}}+(N+1)^2$) in our results.
				
				In Figure~\ref{fig:const-stepsize} $L^N(\overrightarrow{W}(k))$ is plotted versus the iteration number. For the plot \ref{fig:conv_suggested_delta_etak_cons_0,000773_0,000129_0,0000395} the stepsize is chosen to exactly meet the upper bound in \eqref{bound:eta:numerical} (with $\delta =M^{2/N}/N^3$), resulting for this experiment in the values $\eta =7.73\cdot 10^{-4}$, $\eta = 1.29\cdot 10^{-4}$ and $\eta = 3.91\cdot 10^{-5}$ for depth $2,3$ and $5$, respectively.
				For the plot  \ref{fig:conv_const_suggested_delta_eta=10^(-5)}, the step size $\eta$ is chosen somewhat smaller than the upper bound in  \eqref{bound:eta:numerical}, while for plots \ref{fig:conv_const_suggested_delta_eta=10^(-3)} and \ref{fig:conv_const_suggested_delta_eta=0.05} the bound \eqref{bound:eta:numerical} is not satisfied. Since we observe convergence in plot  \ref{fig:conv_const_suggested_delta_eta=10^(-3)}, this suggests that the bound of Theorem~\ref{conv_critical_pt} may not be entirely sharp. 
				But increasing the step size beyond a certain value leads to divergence as suggested by plot~\ref{fig:stepsize_l=0.05stepsize_r=0.1suggested_delta_Page6_modified_etak_const_gamma=0.2}, so that some bound on the step size is necessary (see also \cite[Lemma A.1]{chou2021} for a necessary condition in a special case).

				\begin{figure}[t]
					\begin{subfigure}{0.45\textwidth}
						\centering\includegraphics[width=0.99\textwidth]{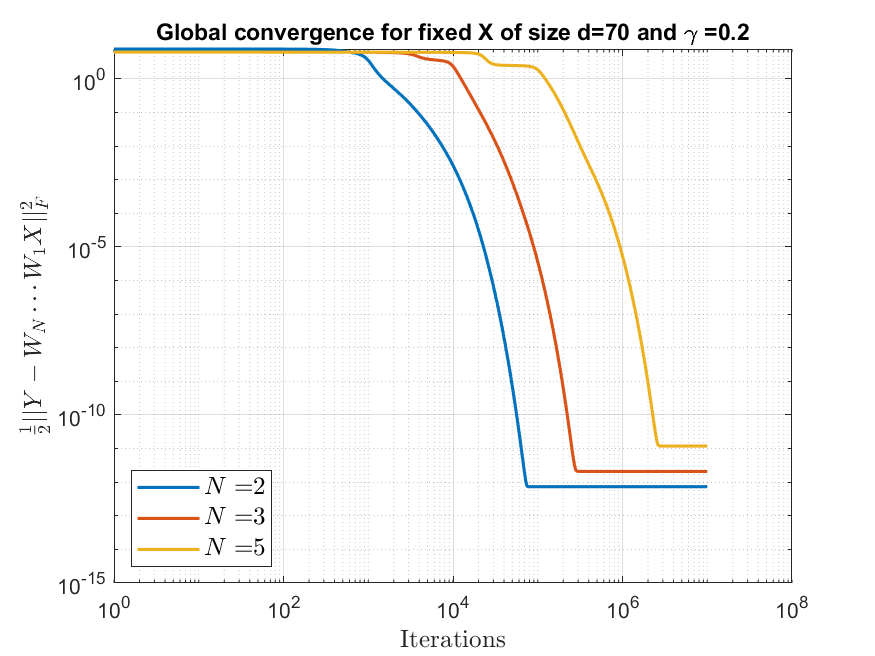}
						\caption{
							Step sizes $\eta_k$ in \eqref{etak_with_from_g} with $\gamma = 0.2$ and $a_1, a_2$ in \eqref{a1:a2:choice}}
						\label{fig:conv_const_suggested_delta_etak_gamma=0.2}
					\end{subfigure}\hfill\begin{subfigure}{0.45\textwidth}
						\centering\includegraphics[width=0.99\textwidth]{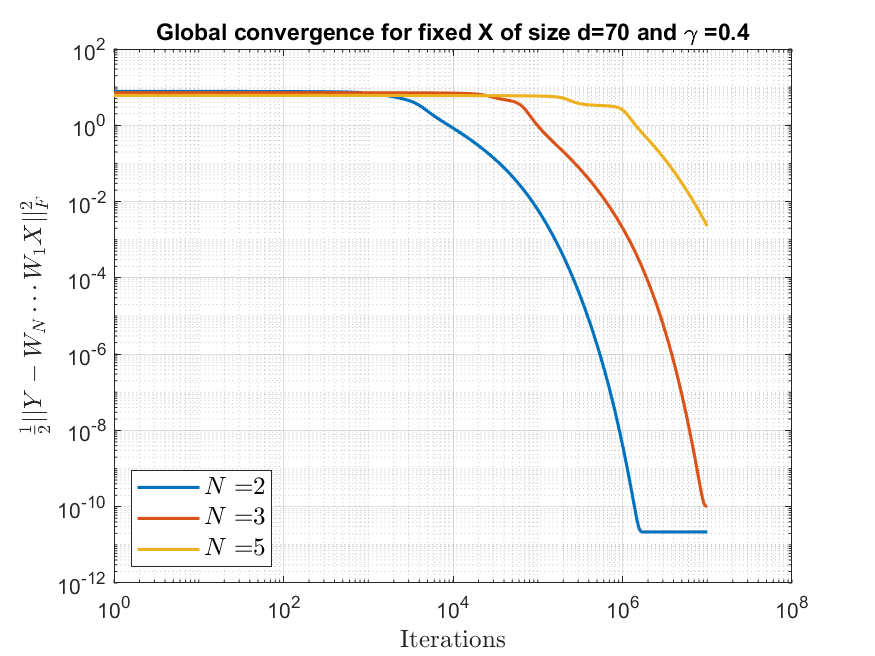}
						\caption{Step sizes $\eta_k$ in \eqref{etak_with_from_g} with $\gamma = 0.4$ and $a_1, a_2$ in \eqref{a1:a2:choice}}
					\label{fig:conv_const_suggested_delta_etak_gamma=0.4}
				\end{subfigure}
				\begin{subfigure}{0.45\textwidth}
					\centering\includegraphics[width=0.99\textwidth]{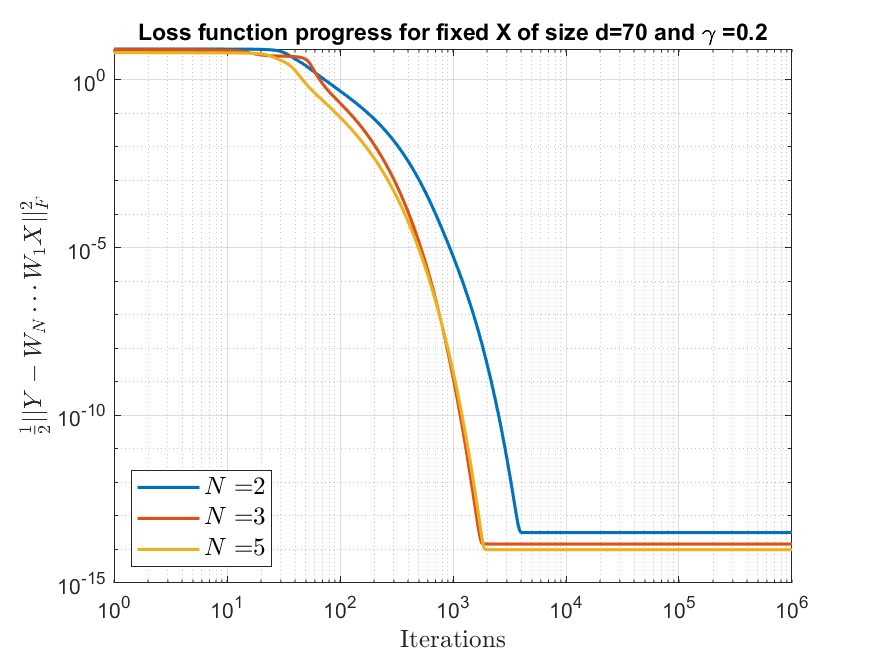}
					\caption{Step sizes $\eta_k$ in \eqref{etak_with_from_g} with $\gamma = 0.2$, $a_1=0.1$, $a_2=0.01$}
				\label{fig:stepsize_l=0.1stepsize_r=0.01suggested_delta_Page6_modified_etak_const_gamma=0.2}
			\end{subfigure}\hfill\begin{subfigure}{0.45\textwidth}\centering\includegraphics[width=0.99\textwidth]{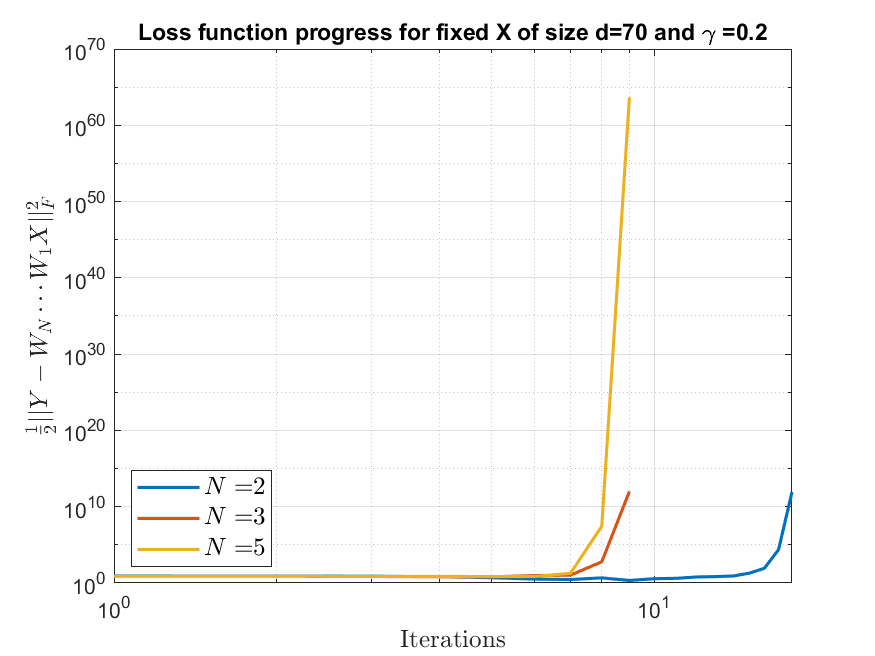}
				\caption{
					Step sizes $\eta_k$ in \eqref{etak_with_from_g} with $\gamma = 0.2$, $a_1=0.05$, $a_2=0.1$}
			\label{fig:stepsize_l=0.05stepsize_r=0.1suggested_delta_Page6_modified_etak_const_gamma=0.2}
		\end{subfigure}
		\caption{Gradient descent with decreasing step sizes $\eta_k$ as in \eqref{etak_with_from_g}}\label{fig:decreasing-stepsize}
\end{figure}

In our second set of experiments we use a sequence of step sizes $\eta_k$ that converges to zero at various speeds. For some decay rate $\gamma \geq 0$ and some constants $a_1, a_2$ we set
\begin{align}\label{etak_with_from_g}
	\eta_k=\min\left\{a_1 ,\frac{a_2}{(k+1)^{\gamma}} \right\}&\quad
	\gamma\geq 0,\quad \mbox{for all } k \in \mathbb{N}.
\end{align}
The upper bound of Theorem~\ref{conv_critical_pt} is satisfied for (see also the beginning of the  proof of Proposition \ref{prop1})
\begin{align}\label{a1:a2:choice}
	a_1= a_2 = \frac{2(1-\sigma)}{B_\delta},
	\quad\sigma = \frac{4 L^N(\overrightarrow{W}(0))}{4 L^N(\overrightarrow{W}(0)) +  \delta B_\delta}. 
\end{align}
Again, we choose $\delta =\frac{1}{N^3} M^{\frac{2}{N}}$, which corresponds to the choice of $\delta$ using the bound given in Remark~\ref{Wj-bound-balancedness-improved} when testing with these values for $a_1$ and $a_2$.

The plots in Figure~\ref{fig:decreasing-stepsize}  illustrate the convergence behavior for various choices of the constants $a_1$, $a_2$ and decay rate $\gamma$ in \eqref{etak_with_from_g}, for $N=2,3,5$.
Plot \ref{fig:conv_const_suggested_delta_etak_gamma=0.2} and \ref{fig:conv_const_suggested_delta_etak_gamma=0.4} both show convergence for the choices $a_1, a_2$ in \eqref{etak_with_from_g} and for $\gamma = 0.2$ and $\gamma = 0.4$ respectively, leading to step sizes satisfying the condition of Theorem~\ref{conv_critical_pt}.
(In these experiments, the resulting values of $a_1=a_2$ are $a_1 = 7.73 \cdot 10^{-4}$ for $N=2$, $a_1 = 1.29 \cdot 10^{-4}$ for $N=3$ and $a_1 = 3.91 \cdot 10^{-5}$ for $N=5$.) Comparing the two plots, as well as with the plots for constant step size in Figure~\ref{fig:const-stepsize}, shows that fast decay of the step size leads to slower convergence of gradient descent, as expected. Note that we observe that larger values of $\gamma$ are possible, but will further slow down convergence, so that we decided to omit the corresponding experiments here.

Plot~\ref{fig:stepsize_l=0.1stepsize_r=0.01suggested_delta_Page6_modified_etak_const_gamma=0.2} shows convergence for a decay rate of $y = 0.2$ even though the constants $a_1$ and $a_2$ are such that $\eta_k$ does not satisfy the bound of Theorem~\ref{conv_critical_pt} for all $k$, while further increasing the value of $a_2$ leads to divergence as illustrated 
in Plot~\ref{fig:stepsize_l=0.05stepsize_r=0.1suggested_delta_Page6_modified_etak_const_gamma=0.2}.

				\section*{Acknowledgements}
				
				All authors acknowledge funding by DAAD (German Foreign Exchange Service) through the project \textit{ Understanding stochastic gradient descent in deep learning} (grant no: 57417829).

				\bibliographystyle{abbrv}
				\bibliography{GradDescRef}
				
			\end{document}